\documentclass{article}

\usepackage[final]{neurips_2019}

\usepackage[utf8]{inputenc} 
\usepackage[T1]{fontenc}    
\usepackage{hyperref}       
\usepackage{url}            
\usepackage{booktabs}       
\usepackage{amsfonts}       
\usepackage{nicefrac}       
\usepackage{microtype}      

\usepackage{etoolbox}
\makeatletter
\patchcmd\@combinedblfloats{\box\@outputbox}{\unvbox\@outputbox}{}{%
  \errmessage{\noexpand\@combinedblfloats could not be patched}%
}%
\makeatother

\usepackage{graphicx}
\usepackage{amsfonts}
\usepackage{amsmath}
\usepackage[psamsfonts]{amssymb}
\usepackage{amstext}
\usepackage{amsthm}
\usepackage{latexsym}
\usepackage{graphicx}
\usepackage{mathtools}
\usepackage{caption}
\usepackage{subcaption}

\usepackage{algorithm}
\usepackage[noend]{algorithmic}
\usepackage{array}

\usepackage[mathscr]{euscript}
\usepackage{xcolor, colortbl}

\usepackage{wrapfig}

\makeatletter
\newtheorem*{rep@theorem}{\rep@title}
\newcommand{\newreptheorem}[2]{%
\newenvironment{rep#1}[1]{%
 \def\rep@title{#2 \ref{##1}}%
 \begin{rep@theorem}}%
 {\end{rep@theorem}}}
\makeatother

\hypersetup{
  colorlinks   = true, 
  urlcolor     = blue, 
  linkcolor    = blue, 
  citecolor   = blue 
}

\let\Pr\undefined
\def\Rset{\mathbb{R}}

\newcommand{\Rad}{\mathfrak{R}}
\DeclareMathOperator*{\E}{\mathbb{E}}
\def\Pr{\mathbb{P}}

\DeclareMathOperator*{\argmin}{\rm argmin}

\newcommand{\cE}{\mathcal{E}}
\newcommand{\cF}{\mathcal{F}}
\newcommand{\cH}{\mathcal{H}}
\newcommand{\cG}{\mathcal{G}}
\newcommand{\cL}{\mathcal{L}}

\newcommand{\cX}{\mathcal{X}}
\newcommand{\cY}{\mathcal{Y}}
\newcommand{\cZ}{\mathcal{Z}}
\newcommand{\sD}{\mathscr{D}}

\newcommand{\bP}{\mathbb{P}}
\newcommand{\bQ}{\mathbb{Q}}

\newcommand{\bPhi}{\boldsymbol{\Phi}}

\newcommand{\balpha}{\boldsymbol{\alpha}}
\newcommand{\bM}{\boldsymbol{M}}
\newcommand{\bI}{\boldsymbol{I}}

\newcommand{\set}[2][]{#1 \{ #2 #1 \} }

\newcommand{\h}{\widehat}

\newcommand{\e}{\epsilon}

\newcommand{\pz}{\mathbb{P}_{z}}
\newcommand{\pdata}{\mathbb{P}_{r}}
\newcommand{\hdata}{\widehat{\mathbb{P}}_{r}}
\newcommand{\hdatat}{\widehat{\mathbb{P}}_{r}^{t}}
\newcommand{\sdata}{S_{r}}
\newcommand{\sdatat}{S_{r}^{t}}
\newcommand{\pgan}{\mathbb{P}_{\theta}}
\newcommand{\pgant}{\mathbb{P}_{\theta^{t}}}
\newcommand{\pgank}{\mathbb{P}_{\theta_{k}}}
\newcommand{\hgan}{\widehat{\mathbb{P}}_{\theta}}
\newcommand{\hganens}{\widehat{\mathbb{P}}_{\balpha}}
\newcommand{\pganens}{{\mathbb{P}}_{\balpha}}
\newcommand{\hgant}{\widehat{\mathbb{P}}_{\theta}^{t}}
\newcommand{\sgan}{S_{\theta}}
\newcommand{\sgant}{S_{\theta}^{t}}
\newcommand{\disc}{\text{disc}_{\cH,\ell}}

\newcommand{\JS}{\text{JS}}
\newcommand{\KL}{\text{KL}}
\newcommand{\edgan}{\textsc{EDGAN}}
\newcommand{\dgan}{\textsc{DGAN}}

\newtheorem{theorem}{Theorem}
\newreptheorem{theorem}{Theorem}

\newreptheorem{lemma}{Lemma}

\newreptheorem{corollary}{Corollary}
\newtheorem{proposition}[theorem]{Proposition}
\newreptheorem{proposition}{Proposition}
\newtheorem{remark}{Remark}
\newcommand{\ignore}[1]{}
\newcolumntype{P}[1]{>{\centering\arraybackslash}p{#1}}

\title{Learning GANs and Ensembles Using Discrepancy}

\author{
 Ben Adlam \\
 Google Research\\
 New York, NY 10011\\
 \texttt{\small adlam@google.com} \\
 \And
 Corinna Cortes \\
 Google Research \\
 New York, NY 10011 \\
 \texttt{\small corinna@google.com} \\
 \AND
 Mehryar Mohri \\
 Google Research \& CIMS\\
 New York, NY 10012 \\
\texttt{\small mohri@google.com} \\
 \And
 Ningshan Zhang \\
 New York University \\
 New York, NY 10012 \\
 \texttt{\small nzhang@stern.nyu.edu} \\
}

\begin{document}

\maketitle
\begin{abstract}
  Generative adversarial networks (GANs) generate data based on
  minimizing a divergence between two distributions. The choice of
  that divergence is therefore critical. We argue that the divergence
  must take into account the hypothesis set and the loss function used
  in a subsequent learning task, where the data generated by a GAN
  serves for training. Taking that structural information into account
  is also important to derive generalization guarantees. Thus, we
  propose to use the \emph{discrepancy} measure, which was originally
  introduced for the closely related problem of domain adaptation and
  which precisely takes into account the hypothesis set and the loss
  function. We show that discrepancy admits favorable properties for
  training GANs and prove explicit generalization guarantees.  We
  present efficient algorithms using discrepancy for two tasks:
  training a GAN directly, namely \dgan, and mixing previously trained
  generative models, namely \edgan.  \ignore{In particular, we show
    that learning an ensemble of generators by minimizing discrepancy
    is a convex optimization problem, thereby benefitting from strong
    convergence guarantees.}  Our experiments on toy examples and
  several benchmark datasets show that \dgan\ is competitive with
  other GANs and that \edgan\ outperforms existing GAN ensembles, such
  as AdaGAN.
\end{abstract}

\section{Introduction}
\label{sec:intro}

Generative adversarial networks (GANs) consist of a family of methods
for unsupervised learning. A GAN learns a generative model that can
easily output samples following a distribution $\pgan$, which aims to
mimic the real data distribution $\pdata$.  The parameter $\theta$ of
the generator is learned by minimizing a divergence between $\pdata$
and $\pgan$, and different choices of this divergence lead to
different GAN algorithms: the Jensen-Shannon divergence gives the
standard GAN \citep{goodfellow2014generative,salimans2016improved},
the Wasserstein distance gives the WGAN
\citep{arjovsky2017wasserstein,gulrajani2017improved}, the squared
maximum mean discrepancy gives the MMD GAN
\citep{li2015generative,dziugaite2015training,li2017mmd}, and the
$f$-divergence gives the $f$-GAN \citep{nowozin2016fgan}, just to name
a few.  There are many other GANs that have been derived using other
divergences in the past, see \citep{goodfellow2016nips} and
\citep{creswell2018generative} for more extensive studies.

The choice of the divergence seems to be critical in the design of a
GAN. But, how should that divergence be selected or defined? We argue
that its choice must take into consideration the structure of a
learning task and include, in particular, the hypothesis set and the
loss function considered.  In contrast, divergences that ignore the
hypothesis set typically cannot benefit from any generalization
guarantee (see for example \cite{arora2017generalization}).
The loss function is also crucial: while many GAN applications aim to
generate synthetic samples indistinguishable from original ones, for
example images \citep{stylegan,highfidelitygan} or Anime characters
\citep{anime}, in many other applications, the generated samples are
used to improve subsequent learning tasks, such as data augmentation
\citep{dataaugmentgan}, improved anomaly detection \citep{anomalygan},
or model compression \citep{modelcompressiongan}. Such subsequent
learning tasks require optimizing a specific loss function applied to
the data. Thus, it would seem beneficial to explicitly incorporate
this loss in the training of a GAN.

A natural divergence that accounts for both the loss function and the
hypothesis set is the \emph{discrepancy} measure introduced by
\cite{mansour2009domain}.\ignore{ (see also
  \citep{CortesMohriMunozMedina2019,CortesMohri2014}). The
  \emph{$d_A$-distance} of \cite{ben2007analysis} is a special case of
  discrepancy in the case of the zero-one loss. } Discrepancy plays a
key role in the analysis of domain adaptation, which is closely
related to the GAN problem, and other related problems such as
drifting and time series prediction
\citep{MohriMuniozMedina2012,KuznetsovMohri2015}. Several important
generalization bounds for domain adaptation are expressed in terms of
discrepancy \citep{mansour2009domain,CortesMohri2014,ben2007analysis}.
We define discrepancy in Section~\ref{sec:discrepancy} and give
examples illustrating the benefit of using discrepancy to measure the
divergence between distributions.

In this work, we design a new GAN technique, \emph{discrepancy GAN}
(\dgan), that minimizes the discrepancy between $\pgan$ and
$\pdata$. By training GANs with discrepancy, we obtain theoretical
guarantees for subsequent learning tasks using the samples it
generates. We show that discrepancy is continuous with respect to the
generator's parameter $\theta$, under mild conditions, which makes
training \dgan\ easy.  Another key property of the discrepancy is that
it can be accurately estimated from finite samples when the hypothesis
set admits bounded complexity. This property does not hold for popular
metrics such as the Jensen-Shannon divergence and the Wasserstein
distance.

Moreover, we propose to use discrepancy to learn an ensemble of
pre-trained GANs, which results in our \edgan\ algorithm. By
considering an ensemble of GANs, one can greatly reduce the problem of
missing modes that frequently occurs when training a single GAN.  We
show that the discrepancy between the true and the ensemble
distribution learned on finite samples converges to the discrepancy
between the true and the optimal ensemble distribution, as the sample
size increases.  We also show that the \edgan\ problem can be
formulated as a convex optimization problem, thereby benefiting from
strong convergence guarantees. Recent work of
\cite{tolstikhin2017adagan}, \cite{arora2017generalization},
\cite{ghosh2018multi} and \cite{hoang2018mgan} also considered mixing
GANs, either motived by boosting algorithms such as AdaBoost, or by
the minimax theorem in game theory.  These algorithms train multiple
generators and learn the mixture weights simultaneously, yet none of
them explicitly optimizes for the mixture weights once the multiple
GANs are learned, which can provide additional improvement as
demonstrated by our experiments with \edgan.

The term ``discrepancy'' has been previously used in the GAN
literature under a different definition.  The \emph{squared maximum
  mean discrepancy (MMD)}, which was originally proposed by
\cite{gretton2012kernel}, is used as the distance metric for training
MMD GAN \citep{li2015generative,dziugaite2015training,li2017mmd}.  MMD
between two distributions is defined with respect to a family of
functions $\cF$, which is usually assumed to be a reproducing kernel
Hilbert space (RKHS) induced by a kernel function, but MMD does not
take into account the loss function.  LSGAN \citep{mao2017least} also
adopts the squared loss function for the discriminator, and as we do
for \dgan.  \cite{feizi2017understanding,deshpande2018generative}
consider minimizing the quadratic Wasserstein distance between the
true and the generated samples, which involves the squared loss
function as well.  However, their training objectives are vastly
different from ours.  Finally, when the hypothesis set is the family
of linear functions with bounded norm and the loss function is the
squared loss, \dgan\ coincides with the objective sought by McGAN
\citep{mroueh17amcgan}, that of matching the empirical covariance
matrices of the true and the generated distribution. However, McGAN
uses nuclear norm while \dgan\ uses spectral norm in that case.

The rest of this paper is organized as follows. In
Section~\ref{sec:discrepancy}, we define discrepancy and prove that it
benefits from several favorable properties, including continuity with
respect to the generator's parameter and the possibility of accurately
estimating it from finite samples.  In Section~\ref{sec:algorithm}, we
describe our discrepancy GAN (\dgan) and ensemble discrepancy GAN
(\edgan) algorithms with a discussion of the optimization solution and
theoretical guarantees.  We report the results of a series of
experiments (Section~\ref{sec:experiments}), on both toy examples and
several benchmark datasets, showing that \dgan\ is competitive with
other GANs and that \edgan\ outperforms existing GAN ensembles, such
as AdaGAN.

\section{Discrepancy}
\label{sec:discrepancy} 

Let $\pdata$ denote the real data distribution on $\cX$, which,
without loss of generality, we can assume to be
$\cX = \set{ x \in \Rset^d \colon \| x \|_2 \leq 1}$.  A GAN generates
a sample in $\cX$ via the following procedure: it first draws a random
noise vector $z \in \cZ$ from a fixed distribution $\pz$, typically a
multivariate Gaussian, and then passes $z$ through the generator
$g_{\theta} \colon \cZ \to \cX$, typically a neural network
parametrized by $\theta \in \Theta$.  Let $\pgan$ denote the resulting
distribution of $g_{\theta}(z)$.  Given a distance metric
$d(\cdot,\cdot)$ between two distributions, a GAN's learning objective
is to minimize $d(\pdata,\pgan)$ over $\theta \in \Theta$.

In Appendix~\ref{app:ganintro}, we present and discuss two instances
of the distance metric $d(\cdot,\cdot)$ and two widely-used GANs: the
Jensen-Shannon divergence for the standard GAN
\citep{goodfellow2014generative}, and the Wasserstein distance for
WGAN \citep{arjovsky2017wasserstein}.  Furthermore, we show that
Wasserstein distance can be viewed as discrepancy without considering
the hypothesis set and the loss function, which is one of the reasons
why it cannot benefit from theoretical guarantees.
In this section, we describe the discrepancy measure and motivate its
use by showing that it benefits from several important favorable
properties.

Consider a hypothesis set $\cH$ and a symmetric loss function
$\ell\colon \cY \times \cY \to \Rset$, which will be used in future
supervised learning tasks on the true (and probably also the
generated) data.  Given $\cH$ and $\ell$, the discrepancy between two
distributions $\bP$ and $\bQ$ is defined by the following:
\begin{equation}
\label{eq:disc} \disc (\bP,\bQ) =
\sup_{h,h^{\prime}\in\cH}
\Big|\E_{x\sim\bP}\big[\ell\big(h(x),h^{\prime}(x)\big)\big]-
\E_{x\sim\bQ}\big[\ell\big(h(x),h^{\prime}(x)\big)\big]\Big|.
 \end{equation} 
Equivalently, 
let $\ell_{\cH}=\big\{ \ell\big(h(x), h'(x)\big)\colon h,h'\in\cH\big\}$
be the family of discriminators induced by $\ell$ and $\cH$,
then, the discrepancy can be written as $\disc(\bP, \bQ) =
\sup_{f\in\ell_{\cH}} \big|\E_{\bP} [f(x)] - \E_{\bQ} [f(x)]\big|$.

How would subsequent learning tasks benefit from samples generated by
GANs trained with discrepancy?  We show that, under mild conditions,
any hypothesis performing well on $\pgan$ (with loss function $\ell$)
is guaranteed to perform well on $\pdata$, as long as the discrepancy
$\disc(\pgan, \pdata)$ is small.
\begin{theorem}\label{th:da}
    Assume the true labeling function $f\colon \cX \to \cY$ is 
    contained in the hypothesis set $\cH$. Then, for any hypothesis $h\in\cH$,
    \[
        \E_{x\sim \pdata} [\ell(h, f)] \leq \E_{x\sim\pgan} [\ell(h, f)] +
        \disc(\pgan, \pdata).
    \]
\end{theorem}
Theorem~\ref{th:da} suggests that the learner can learn a model using
samples drawn from $\pgan$, whose generation error on $\pdata$ is
guaranteed to be no more than its generation error on $\pgan$ plus the
discrepancy, which is minimized by the algorithm.  The proof uses the
definition of discrepancy. Due to space limitation, we provide all the
proofs in Appendix~\ref{app:proof}.

\subsection{Hypothesis set and loss function} 

We argue that discrepancy is more favorable than Wasserstein distance
measures, since it makes explicit the dependence on loss function and
hypothesis set.  We consider two widely used learning scenarios: 0-1
loss with linear separators, and squared loss with Lipschitz
functions.

\paragraph{0-1 Loss, Linear Separators}
Consider the two distributions on $\Rset^2$ illustrated in
Figure~\ref{fig:wass_vs_disc_sub1}: $\bQ$ (filled circles $\bullet$)
is equivalent to $\bP$ (circles $\circ$), but with all points shifted
to the right by a small amount $\e$.  Then, by the definition of
Wasserstein distance, $ W(\bP, \bQ) = \e$, since to transport $\bP$ to
$\bQ$, one just need to move each point to the right by $\e$.  When
$\e$ is small, WGAN views the two distributions as close and thus
stops training.  On the other hand, when $\ell$ is the 0-1 loss and
$\cH$ is the set of linear separators, $\disc (\bP, \bQ)=1$, which is
achieved at the $h,h'$ as shown in Figure~\ref{fig:wass_vs_disc_sub1},
with $\E_{\bP} [1_{h(x) \ne h'(x)} ]= 1$ and
$\E_{\bQ} [1_{h(x) \ne h'(x)} ]= 0$.  Thus, \dgan\ continues training
to push $\bQ$ towards $\bP$.

The example above is an extreme case where $\bP$ and $\bQ$ are
separable.  In more practical scenarios, the domain of the two
distributions may overlap significantly, as illustrated in
Figure~\ref{fig:wass_vs_disc_sub2}, where $\bP$ is in red and $\bQ$ is
in blue, and the shaded areas contain 95\% probably mass.  Again,
$\bQ$ equals $\bP$ shifting to the right by $\e$ and thus
$W(\bP,\bQ)=\e$.  Since the non-overlapping area has a sizable
probability mass, the discrepancy between $\bP$ and $\bQ$ is still
large, for the same reason as for Figure~\ref{fig:wass_vs_disc_sub1}.

These examples demonstrate the importance of taking hypothesis sets
and loss functions into account when comparing two distributions: even
though two distributions appear geometrically ``close'' under
Wasserstein distance, a classifier trained on one distribution may
perform poorly on another distribution. According to
Theorem~\ref{th:da}, such unfortunate behaviors are less likely to
happen with $\disc$.  \ignore{In fact, in the consistent case where
  the true labeling function $f\colon \cX \to \cY$ is contained in the
  hypothesis set $\cH$, a hypothesis performing well on $\bQ$ is
  guaranteed to perform well on $\bP$ when $\disc(\bP, \bQ)$ is small,
  since
  $\E_{x\sim \bP} \ell(h, f) \leq \E_{x\sim\bQ} \ell(h, f) +
  \disc(\bP, \bQ)$, which holds by the definition of $\disc$.}

\begin{figure} 
  \centering
  \begin{subfigure}{.5\textwidth} \centering
    \includegraphics[width=.7\linewidth]{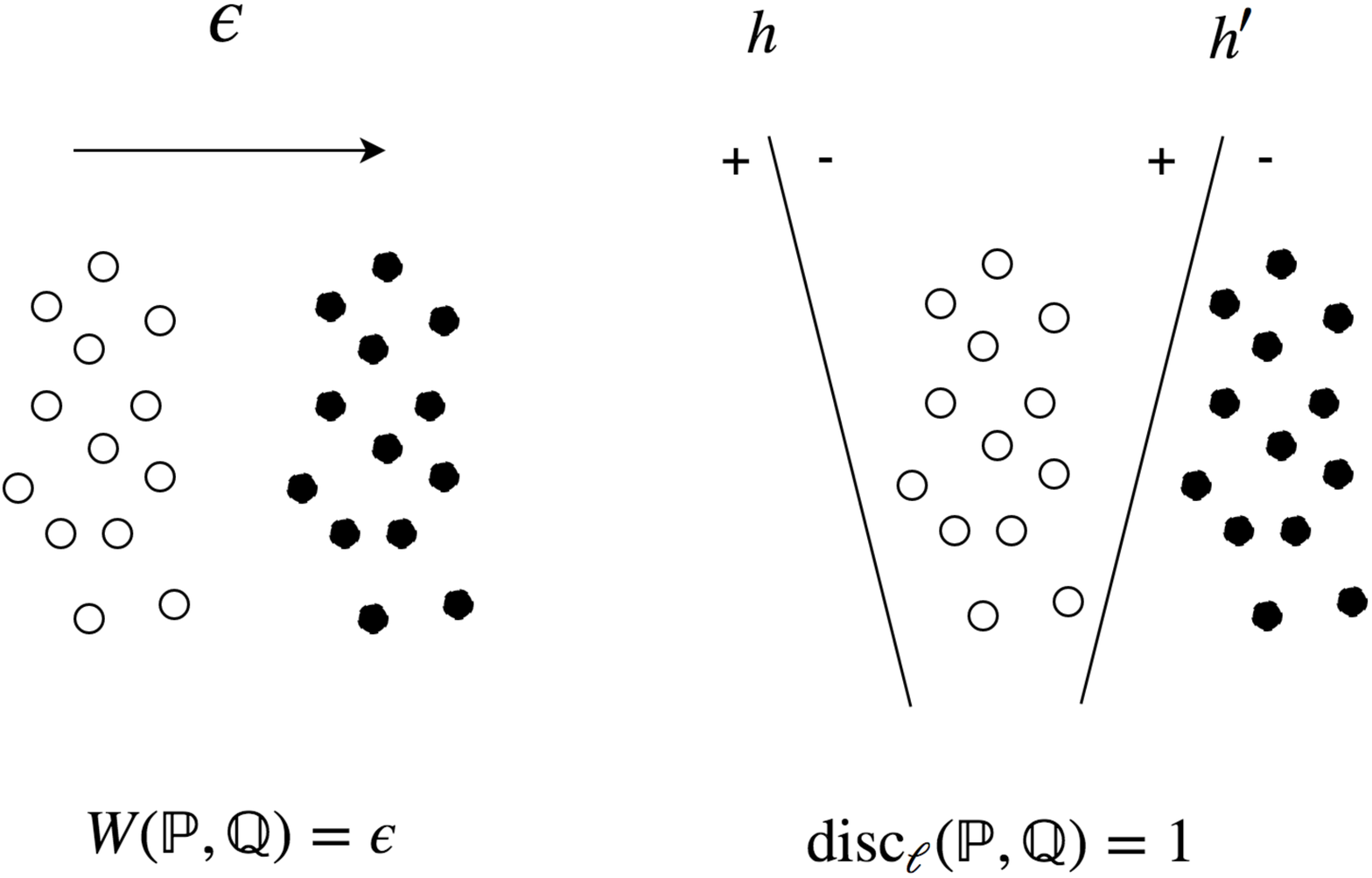}
          \vskip -0.15in
    \caption{Non-overlapping distributions, $\bP$: $\set{\circ}$,
      $\bQ$: $\set{\bullet}$.}
    \label{fig:wass_vs_disc_sub1}
  \end{subfigure}%
  \begin{subfigure}{.5\textwidth} \centering
    \includegraphics[width=.58\linewidth]{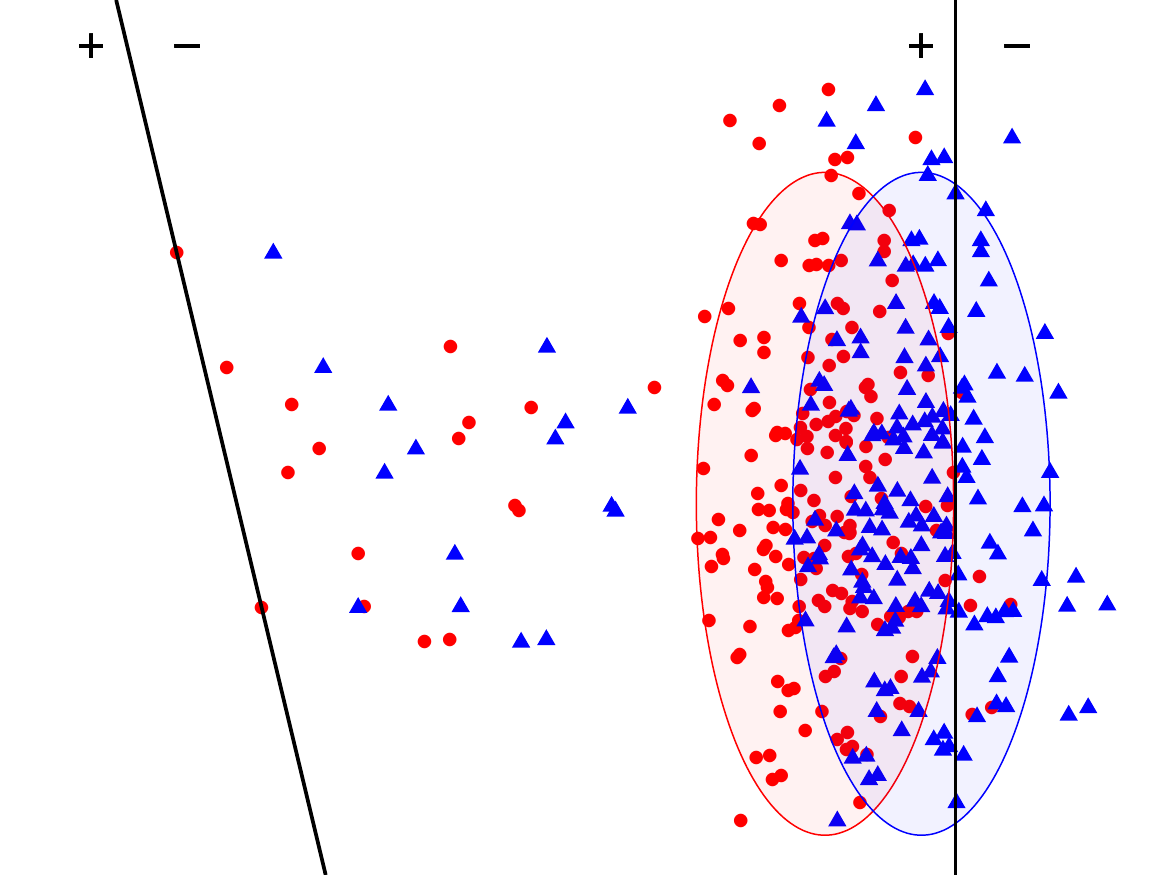}
              \vskip 0.15in
    \caption{Overlapping distributions, $\bP$:
      $\set{\textcolor{red}{\text{red}}}$, $\bQ$:
      $\set{\textcolor{blue}{\text{blue}}}$.}
    \label{fig:wass_vs_disc_sub2}
  \end{subfigure}
  \caption{Distributions $\bP$ and $\bQ$ may appear ``close'' under
    Wasserstein distance, but the discrepancy between the two is still
    large, where the discrepancy is defined by 0-1 loss and linear
    separators.}
  \label{fig:wass_vs_disc}
\end{figure}

\paragraph{Squared Loss, Lipschitz Functions} Next, we consider the
squared loss and the hypothesis set of 1-Lipschitz functions
$\cH = \set{h\colon |h(x)-h(x') | \leq \| x-x'\|_2, \forall
  x,x'\in\cX}$, then
$\ell_\cH=\{ [h(x) - h'(x)]^2\colon h,h'\in\cH\}$.  We can show that
$\ell_\cH$ is a subset of 4-Lipschitz functions on $\cX$.  Then, by
the definition of discrepancy and Wasserstein distance,
$\disc(\bP,\bQ)$ is comparable to $W(\bP,\bQ)$:
\begin{align*} 
    \disc (\bP, \bQ) 
    =  \sup_{f\in \ell_\cH } \E_{\bP} \big[f(x)\big] - \E_{\bQ} \big[f(x)\big]
    \leq \sup_{f\colon
    4\text{-Lipschitz}} \E_{\bP} \big[f(x)\big] - \E_{\bQ} \big[f(x)\big]
= 4W(\bP, \bQ).
\end{align*}  
However, the inequality above can be quite loose since, depending on
the hypothesis set, $\ell_\cH$ may be only a small subset of all
$4$-Lipschitz functions.  For instance, when $\cH$ is the set of
linear functions with norm bounded by one, then
$\ell_\cH = \{(w^Tx)^2 \colon \| w\| \leq 2\}$, which is a
significantly smaller set than the family of all $4$-Lipschitz
functions.  Thus, $\disc(\bP,\bQ)$ could potentially be a tighter
measure than $W(\bP,\bQ)$, depending on $\cH$.

\subsection{Continuity and estimation}
\label{subsec:cont_est}

In this section, we discuss two favorable properties of
discrepancy: its continuity under mild assumptions with respect 
to the generator's parameter $\theta$, a property shared with 
the Wasserstein distance, and the fact that it can be
accurately estimated from finite samples, which 
does not hold for either the Jensen-Shannon or the
Wasserstein distance.
The continuity property is summarized in the following  theorem.

\begin{theorem}
\label{th:disc_cont} 
Let $\cH = \set{h \colon \cX \to \cY}$ be a family of $\mu$-Lipschitz
functions and assume that the loss function $\ell$ is continuous and
symmetric in its arguments, and is bounded by $M$.  Assume further
that $\ell$ admits the triangle inequality, or that it can be written
as $\ell(y, y') = f(|y - y'|)$ for some Lipschitz function $f$.
Assume that $g_\theta \colon \cZ \to \cX$ is continuous in $\theta$.
Then, $\disc(\pdata, \pgan)$ is continuous in $\theta$.
\end{theorem}

The assumptions of Theorem~\ref{th:disc_cont} are easily satisfied in
practice, where $h \in \cH$ and $g_\theta$ are neural networks whose
parameters are limited within a compact set, and where the loss
function can be either the $\ell_1$ loss, $\ell(y, y') = |y-y'|$, or
the squared loss, $\ell(y,y') = (y - y')^2$.  If the discrepancy is
continuous in $\theta$, then, as the sequence of parameters $\theta_t$
converges to $\theta^*$, the discrepancy also converges:
$|\disc(\pdata,\mathbb{P}_{\theta_t}) -
\disc(\pdata,\mathbb{P}_{\theta^*})| \to 0$, which is a desirable
property for training \dgan.  The reader is referred to
\cite{arjovsky2017wasserstein} for a more extensive discussion of the
continuity properties of various distance metrics and their effects on
training GANs.

Next, we show that discrepancy can be accurately estimated from finite
samples.  Let $\sdata$ and $\sgan$ be i.i.d.\ samples drawn from
$\pdata$ and $\pgan$ with $|\sdata| = m$ and $|\sgan| = n$, and let
$\hdata$ and $\hgan$ be the empirical distributions induced by 
$\sdata$ and $\sgan$, respectively.  
Recall that the empirical Radmacher complexity of a
hypothesis set $\cG$ on sample $S$ of size $m$ is defined by:
$\h \Rad_{S}(\cG) = \frac{2}{m} \E_{\sigma} \big[\sup_{g\in\cG}
\sum_{i = 1}^{m} \sigma_i g(x_i) \big]$, where
 $\sigma _{1},\sigma _{2},\dots ,\sigma _{m}$ are   
i.i.d.\ random variables with $\Pr(\sigma_i=1) =\Pr(\sigma_i=-1)=1/2$.
The empirical Radmacher complexity measures the complexity
of the hypothesis set $\cG$.
The next theorem presents  the learning guarantees of discrepancy.

\begin{theorem}
\label{th:disc_gen} 
Assume the loss is bounded, $\ell\leq M$.  
For any $\delta > 0$, with probability at least $1 - \delta$ over the
draw of $\sdata$ and $\sgan$, 
\[ \big|\disc(\pdata,\pgan)-\disc(\hdata,\hgan)\big|
\leq\widehat{\Rad}_{\sdata}(\ell_{\cH})+\widehat{\Rad}_{\sgan}(\ell_{\cH})
+3M\Big(\sqrt{\tfrac{\log(4/\delta)}{2m}}+\sqrt{\tfrac{\log(4/\delta)}{2n}}\,\Big).
\] 
Furthermore, when the loss function $\ell(h,h')$ is a $q$-Lipschitz
function of $h-h'$, we have
\[ 
\big| \disc(\pdata, \pgan) - \disc(\hdata, \hgan)
  \big| \leq 4q \Big( \widehat{\Rad}_{\sdata}(\cH) +
  \widehat{\Rad}_{\sgan}(\cH) \Big) + 
 3M \Big(  \sqrt{\tfrac{\log(4/\delta)}{2m}} + \sqrt{\tfrac{\log(4/\delta)}{2n}} \Big).
\]
\end{theorem} 
In the rest of this paper, we will consider the squared 
loss $\ell(y,y') = (y-y')^2$, which is bounded and $2$-Lipschitz
when $|h(x)| \leq 1$ for all $ h \in\cH$ and $x \in \cX$.
Furthermore,  when $\cH$ is a family of feedforward
neural networks, \cite{cortes2017adanet} provided an explicit upper
bound of $\h \Rad_{S}(\cH) = O(1/\sqrt{m})$ for its complexity,
and thus the right-hand side of the above inequality is in
$O(\frac{1}{\sqrt{m}} + \frac{1}{\sqrt{n}})$. Then, for $m$ and $n$
sufficiently large, the empirical discrepancy is close to the true
discrepancy.  It is important that the discrepancy can be accurately
estimated from finite samples since, when training \dgan, we can only
approximate the true discrepancy with a batch of samples.  In
contrast, the Jensen-Shannon distance and the Wasserstein distance do
not admit this favorable property \citep{arora2017generalization}.

\section{Algorithms}
\label{sec:algorithm}

In this section, we show how to compute the discrepancy and train
\dgan\ for various hypothesis sets and the squared loss.  We also
propose to learn an ensemble of pre-trained GANs via minimizing
discrepancy.  We name this method \edgan, and present its learning
guarantees.

\subsection{\dgan\ algorithm}

Given a parametric family of hypotheses $\cH = \set{h_w\colon w \in W}$, \dgan\
is defined as the following min-max optimization problem:
\begin{equation}
\label{eq:obj_dgan} 
\min_{\theta \in \Theta}
\max_{w,w' \in W}
\Big|\E_{x \sim \pdata}\big[\ell \big(h_w(x), h_{w'}(x)\big) \big]-
\E_{x \sim \pgan}\big[\ell\big(h_w(x), h_{w'}(x)\big) \big]\Big|.
\end{equation}
As with other GANs, \dgan\ is trained by iteratively solving the
min-max problem~\eqref{eq:obj_dgan}.  The minimization over the
generator's parameters $\theta$ can be tackled by standard stochastic
gradient descent (SGD) algorithm with back-propagation. 
The inner maximization problem that computes the discrepancy,
however, can be efficiently solved when $\ell$ is the squared loss function.

We first consider $\cH$ to be the set of linear functions with bounded norm:
$\cH = \set{x \to w^T x \colon \| w \|_2 \leq 1, w \in \mathbb{R}^{{d}} }$. 
Recall the definition of $\sdata$, $\sgan$, $\pdata$ and $\pgan$ from 
Section~\ref{subsec:cont_est}.
\ignore{  Let $\sdata$ and $\sgan$ be i.i.d.\ samples
drawn from $\pdata$ and $\pgan$ with $|\sdata|=m$ and $|\sgan|=n$, and
and let $\hdata$ and $\hgan$ be the empirical distributions induced by
$\sdata$ and $\sgan$, respectively. }
In addition, let $X_r$ and $X_\theta$ denote
the corresponding $m \times d$ and $n \times d$ data matrices, where
each row represents one input.

\begin{proposition}
\label{prop:dgan}
When $\ell$ is the squared loss and $\cH$ the family of linear functions with
norm bounded by $1$, 
$\disc (\hdata, \hgan) = 2 \left \| \frac{1}{n} X_\theta^T
  X_\theta - \frac{1}{m} X_r^T X_r \right \|_2$, where $\| \cdot \|_2$
denotes the spectral norm.
\end{proposition}

Thus, the discrepancy $\disc(\hdata,\hgan)$ equals twice the
largest eigenvalue in absolute value of the data-dependent matrix
$\bM(\theta)=\frac{1}{n}X_\theta^T X_\theta -\frac{1}{m} X_r^T X_r$.
Given $v^*(\theta)$, the corresponding eigenvector at the optimal solution, 
we can then back-propagate the loss
$\disc(\hdata, \hgan) = 2 v^{*T}(\theta) \bM(\theta) v^*(\theta)$ to optimize
for $\theta$.  The maximum or minimum eigenvalue of $\bM(\theta)$ can
be computed in $O(d^2)$ \citep{GolubVanLoan1996}, and the
power method can be used to closely approximate it.

The closed-form solution in Proposition~\ref{prop:dgan} holds for
a family $\cH$ of linear mappings. To generate realistic outcomes with \dgan,
however, we need a more complex hypothesis set $\cH$, such as the
family of deep neural networks (DNN).  Thus, we consider the following
approach: first, we fix a pre-trained DNN classifier, such as the
inception network, and pass the samples through this network to obtain
the last (or any other) layer of embedding $f\colon \cX \to \cE$,
where $\cE$ is the embedding space.  Next, we compute the discrepancy
on the embedded samples with $\cH$ being the family of linear
functions with bounded norm, which admits a closed-form solution
according to Proposition~\ref{prop:dgan}.  In practice, it also makes
sense to train the embedding network together with the generator: let
$f_\zeta$ be the embedding network parametrized by $\zeta$, then
\dgan\ optimizes for both $f_{\zeta}$ and $g_\theta$ .  See
Algorithm~\ref{alg:dgan_one} for a single step of updating \dgan.  In
particular, the learner can either compute $F(\zeta^t, \theta ^t)$
exactly, or use an approximation based on the power method. Note that
when the learner uses a pre-fixed embedding network $f$, the update
step of $\zeta^{t+1}$ can be skipped.

\begin{figure}
\begin{minipage}[t]{0.48\textwidth}
\begin{algorithm}[H]
    \small
    \caption{\textsc{Update} \dgan(${\zeta^{t}}, {\theta^{t}}, \eta$)}
    \label{alg:dgan_one}
    \begin{algorithmic}
        \STATE $X_r \gets [f_{\zeta^{t}}(x_1),\cdots, f_{\zeta^{t}}(x_m)]^T$, where $x_i \sim \pdata$
        \STATE $X_\theta \gets [f_{\zeta^{t}}(x'_1),\cdots, f_{\zeta^{t}}(x'_n)]^T$, where $x'_i \sim \pgant$
        \STATE $F(\zeta^t, \theta ^t) \gets \left\|\frac{1}{n}X_\theta^T X_\theta 
        - \frac{1}{m} X_r^T X_r  \right\|_2$ 
        \STATE Update: $\zeta^{t+1} \gets \zeta^t + \eta \nabla_{\zeta} F(\zeta^t,\theta^t)$
        \STATE Update: $\theta^{t+1} \gets \theta^t - \eta \nabla_{\theta} F(\zeta^t,\theta^t)$
    \end{algorithmic}
\end{algorithm}
\end{minipage}
\hfill
\begin{minipage}[t]{0.48\textwidth}
\begin{algorithm}[H]
    \small
    \caption{\textsc{Update} \edgan(${\balpha^{t}}, f,  \eta$)}
    \label{alg:ensdgan_one}
    \begin{algorithmic}
        \STATE $X_r \gets [f(x_1),\cdots, f(x_{n_r})]^T$, where $x_i \sim \pdata$
        \STATE $X_k \gets [f(x^k_1),\cdots, f(x^k_{n_k})]^T$, where $x^k_i \sim \pgank$
        \STATE $F(\balpha^t) \gets 
        \|  \big(\sum_{k=1}^p \frac{\balpha^t_k }{n_k} X_k ^T X_k\big) - 
        \frac{1}{n_r} X_r ^T X_r\|_2$
        \STATE Update: $\balpha^{t+1} \gets \balpha^t - \eta \nabla_{\balpha} F(\balpha^t)$
    \end{algorithmic}
\end{algorithm}
\end{minipage}
\vskip -.15in
\end{figure}

\subsection{\edgan\ algorithm}

Next, we show that discrepancy provides a principled way of
choosing the ensemble weights to mix pre-trained GANs, 
which admits  favorable convergence guarantees.

Let $g_1, \ldots, g_p$ be $p$ pre-trained GANs. For a given mixture
weight $\balpha = (\alpha_1, \ldots, \alpha_p) \in \Delta$, where
$\Delta = \set{(\alpha_1, \ldots, \alpha_p) \colon \alpha_k \geq 0,
  \sum_{k=1}^p \alpha_k=1}$ is the simplex in $\Rset^p$, we define the
ensemble of $p$ GANs by $g_{\balpha} = \sum_{k=1}^p \alpha_k g_k$. To
draw a sample from the ensemble $g_{\balpha}$, we first sample an
index $k \in [p]=\{1,2,\cdots,p\}$ according to the multinomial
distribution with parameter $\balpha$, and then return a random sample
generated by the chosen GAN $g_k$. We denote by $\pganens$ the
distribution of $g_{\balpha}$.  \edgan\ determines the mixture weight
$\balpha$ by minimizing the discrepancy between $\pganens$ and the
real data $\pdata$: $\min_{\balpha\in\Delta}\disc (\pganens, \pdata)$.

To learn the mixture weight $\balpha$, we approximate the true
distributions by their empirical counterparts: for each $k\in[p]$, we
randomly draw a set of $n_k$ samples from $g_k$, and randomly draw
$n_r$ samples from the real data distribution $\pdata$.  Let $S_k$ and
$S_r$ denote the corresponding set of samples, and let $\h \Pr_k$ and
$\hdata$ denote the induced empirical distributions, respectively. For
a given $\balpha$, let $\hganens=\sum_{k=1}^p \alpha_k \h \Pr_k$ be
the empirical counterparts of $ \pganens$.  We first present a
convergence result for the \edgan\ method, and then describe how to
train \edgan.

Let $\balpha^*$ and $\h \balpha$ be the discrepancy minimizer under
the true and the empirical distributions, respectively:
\begin{align*} \balpha^* & = \argmin_{\balpha \in \Delta} \disc
(\pganens, \pdata),\quad \h \balpha =
\argmin_{\balpha \in \Delta} \disc
(\hganens, \hdata).
\end{align*}
For simplicity, we set $n_k = n_r = n$ for all $k\in[p]$, but the
following result can be easily extended to arbitrary batch size for
each generator.

\begin{theorem}
\label{th:ensdgan_gen} 
For any $\delta > 0$,
with probability at least
$1 - \delta$ over the draw of samples,
\begin{align*} & |\disc(\Pr_{{\h \balpha}}, \pdata) -
\disc (\Pr_{{\balpha^*}}, \pdata) | \leq
2\Big(\h{\Rad}_{S}(\ell_\cH) +
3M\sqrt{{\log[4(p+1)/\delta]}/{2n}}\,\Big),
\end{align*} 
where
$\h{\Rad}_{S}(\ell_\cH) = \max\big\{ \h
{\Rad}_{S_1}(\ell_{\cH}),\ldots, \h {\Rad}_{S_p}(\ell_{\cH}), \h
{\Rad}_{S_r}(\ell_{\cH}) \big\}$.  Furthermore, when the loss function
$\ell(h,h')$ is a $q$-Lipschitz function of $h - h'$, the following
holds with probability $1 - \delta$:
\begin{align*} & |\disc(\Pr_{{\h \balpha}}, \pdata)-
\disc (\Pr_{{\balpha^*}}, \pdata) | \leq 2\Big( 4q \,
\h{\Rad}_S(\cH) + 3M\sqrt{{\log[4(p+1)/\delta]}/{2n}}\,\Big),
\end{align*} 
where
$\h {\Rad}_S(\cH)= \max\big\{ \h {\Rad}_{S_1}({\cH}),\ldots, \h
{\Rad}_{S_p}({\cH}), \h {\Rad}_{S_r}({\cH}) \big\}$.
\end{theorem} 
When $\ell$ is the squared loss and $\cH$ is the family of feedforward
neural networks, the upper bound on $\h{\Rad}_{S}(\ell_{\cH})$ is in
$O(1/\sqrt{n})$.  Since we can generate unlimited samples from each of
the $p$ pre-trained GANs, $n$ can be as large as the number of
available real samples, and thus the discrepancy between the learned
ensemble $\mathbb{P}_{{\h \balpha}}$ and the real data $\pdata$ can be
very close to the discrepancy between the optimal ensemble
$\mathbb{P}_{{\balpha^*}}$ and the real data $\pdata$.  This is a very
favorable generalization guarantee for \edgan, since it suggests that
the mixture weight learned on the training data is guaranteed to
generalize and perform well on the test data, a fact also corroborated
by our experiments.

To compute the discrepancy for \edgan, we again begin with linear
mappings
$\cH=\{ x \to w^T x \colon \| w \|_2 \leq 1, w \in
\mathbb{R}^{{d}}\}$.  For each generator $k\in[p]$, we obtain a
$n_k \times d$ data matrix $X_k$, and similarly we have the
$n_r \times d$ data matrix for the real samples.  Then, by the proof
of Proposition~\ref{prop:dgan}, discrepancy minimization can be
written as
\begin{align}
\label{eq:ens}
    & \min_{\balpha\in\Delta}\, \disc (\hganens, \hdata) \
    = 2 \min_{\balpha\in\Delta}
    \| \bM(\balpha) \|_2, \, \text{with \ }
\bM(\balpha) = \bigg[ \sum_{k=1}^p \frac{\alpha_k }{n_k} X_k ^T X_k\bigg] - 
    \frac{1}{n_r} X_r ^T X_r.
\end{align}

Since $\bM(\balpha)$ and $-\bM(\balpha)$ are affine and thus convex
functions of $\balpha$,
$\| \bM(\balpha) \|_2 = \sup_{\| v \|_2 \leq 1} \big| v^T \bM(\balpha)
v\big|$ is also convex in $\balpha$, as the supremum of
a set of convex functions is convex. Thus, problem~\eqref{eq:ens} is
a convex optimization problem, thereby benefitting from strong
convergence guarantees.

Note that we have
$\| \bM(\balpha) \|_2 = \max\{\lambda_{\text{max}} (\bM(\balpha))
,\lambda_{\text{max}} (-\bM(\balpha)) \}$.  Thus, one way to solve
problem~\eqref{eq:ens} is to cast it as a semi-definite programming
(SDP) problem:
\begin{align*}
    \min_{\balpha, \lambda} \quad & \lambda, \quad\quad
    \text{s.t.}\quad \lambda \bI  - \bM(\balpha)  \succeq 0, \, 
                      \lambda \bI  + \bM(\balpha)  \succeq 0,\,
                      \balpha \geq 0,\, \boldsymbol{1}^T \balpha=1.
\end{align*}
An alternative solution consists of using the power method to
approximate the spectral norm, which is faster when the sample
dimension $d$ is large.
As with \dgan, we can also consider a more complex hypothesis set
$\cH$, by first passing samples through an embedding network $f$, and then
letting $\cH$ be the set of linear mappings on the embedded samples.
Since the generators are already pre-trained for \edgan, we no
longer need to train the embedding network, but instead keep it
fixed. See Algorithm~\ref{alg:ensdgan_one} for one training step
of \edgan.

\section{Experiments}
\label{sec:experiments}

\subsection{\dgan}

In this section, we show that \dgan\ obtains competitive results on
the benchmark datasets MNIST, CIFAR10, CIFAR100, and CelebA (at
resolution $128\times128$). We did unconditional generation and did
not use the labels in the dataset. We trained both the discriminator's
embedding layer and the generator with discrepancy loss as in
Algorithm~\ref{alg:dgan_one}.  Note, we did not attempt to optimize
the architecture and other hyperparameters to get state-of-the-art
results.  We used a standard DCGAN architecture. The main
architectural modification for \dgan\ is that the final dense layer of
the discriminator has output dimension greater than 1 since, in \dgan,
the discriminator outputs an embedding layer rather than a single
score.  The size of this embedding layer is a hyperparameter that can
be tuned, but we refrained from doing so here. See Table~\ref{dgan_architectures} 
in Appendix~\ref{app:moreexps} for DGAN architectures. One important
observation is that larger embedding layers require more samples to
accurately estimate the population covariance matrix of the embedding
layer under the data and generated distributions (and hence the
spectral norm of the difference).

\begin{figure}
    \vskip -.05in
    \centering
    \includegraphics[width=0.6\linewidth]{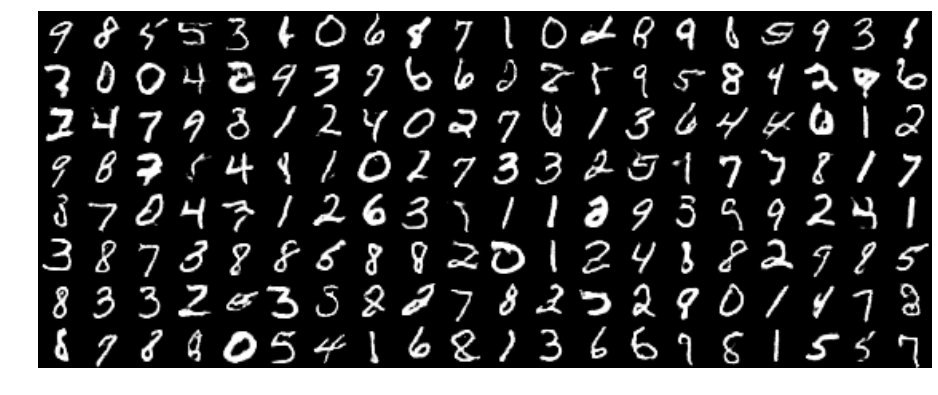} 
    \vskip -.05in
    \caption{Random samples from \dgan\ trained on MNIST. }
    \label{fig:mnist}
\end{figure}

\begin{figure}
    \vskip -.05in
    \centering
\begin{minipage}{.44\textwidth}
  \centering
  \includegraphics[width=.9\linewidth]{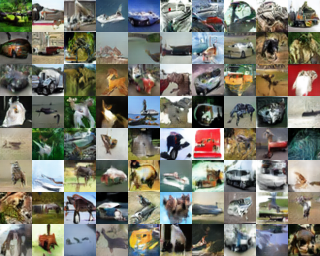} 
\end{minipage}%
\begin{minipage}{.44\textwidth}
  \centering
  \includegraphics[width=.9\linewidth]{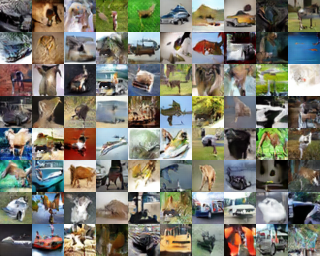} 
\end{minipage}
    \vskip -.05in
    \caption{Random samples from \dgan\ trained on CIFAR10.}
\label{fig:cifar10}
\end{figure}

\begin{figure}
    \vskip -.05in
    \centering
\begin{minipage}{.44\textwidth}
  \centering
  \includegraphics[width=.9\linewidth]{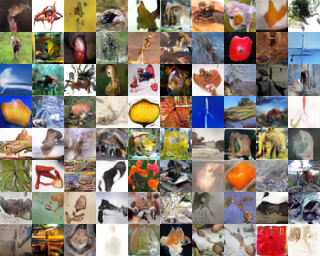} 
\end{minipage}%
\begin{minipage}{.44\textwidth}
  \centering
  \includegraphics[width=.9\linewidth]{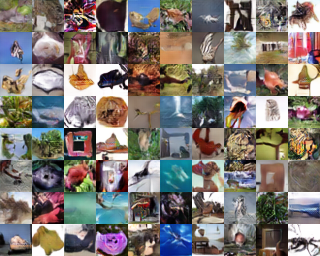} 
\end{minipage}
    \vskip -.05in
\caption{Random samples from \dgan\ trained on CIFAR100. }
\label{fig:cifar100}
\end{figure}

\begin{figure}
    \vskip -.05in
    \centering
\begin{minipage}{.44\textwidth}
  \centering
  \includegraphics[width=.9\linewidth]{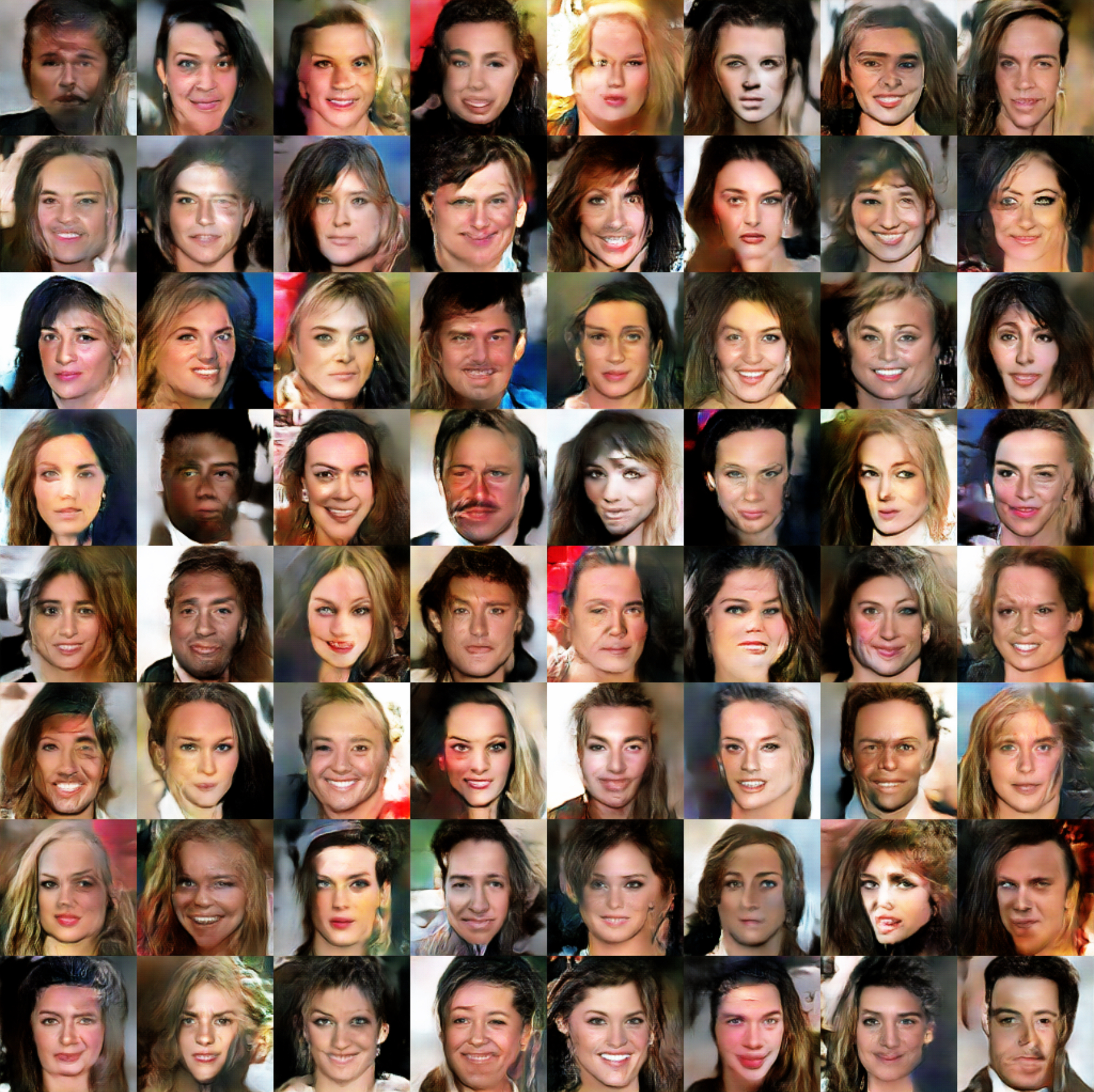} 
\end{minipage}%
\begin{minipage}{.44\textwidth}
  \centering
  \includegraphics[width=.9\linewidth]{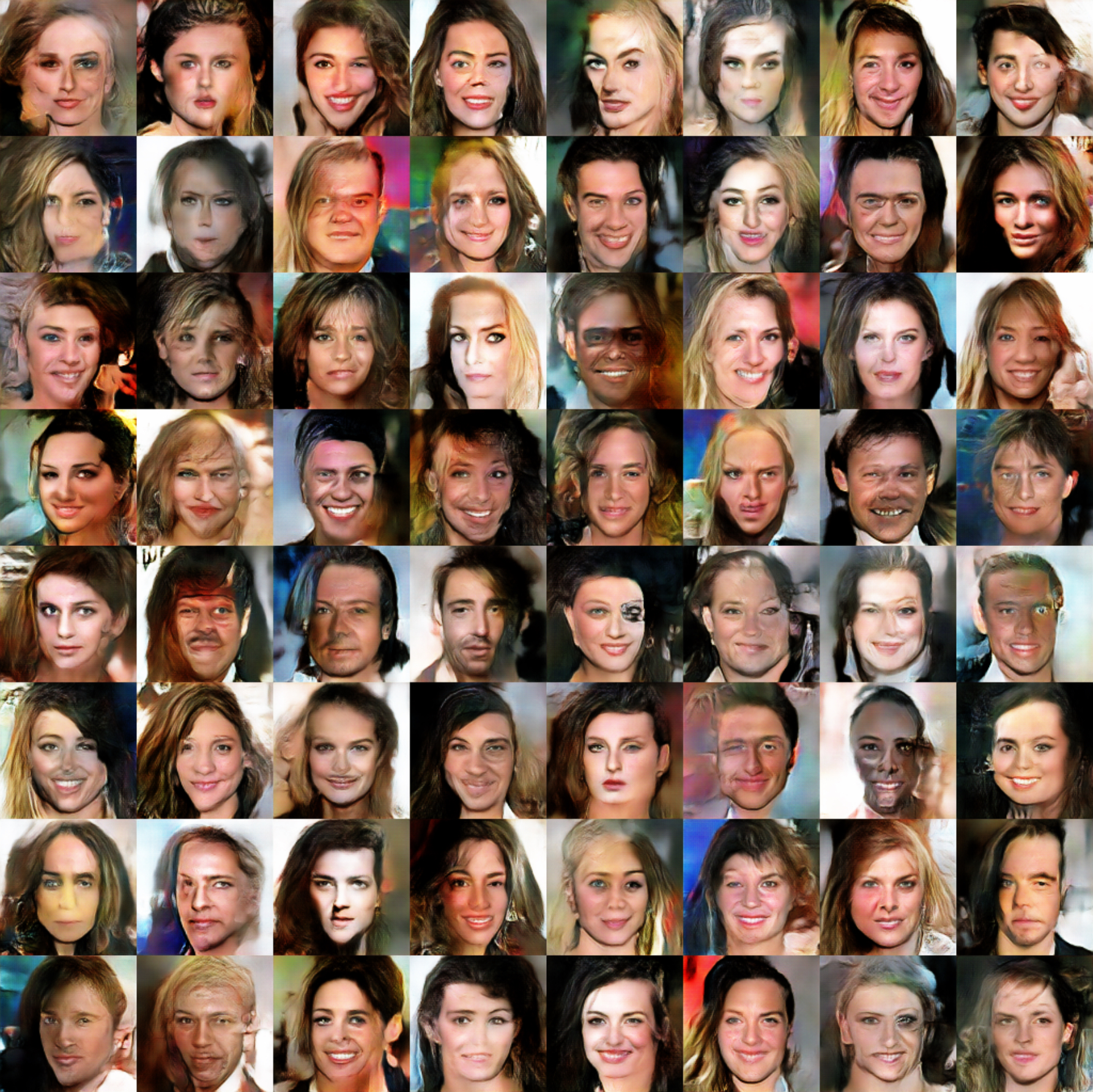} 
\end{minipage}
\vskip -.05in
\caption{Random samples from \dgan\ trained on CelebA at resolution $128\times128$. }
\label{fig:celeba}
\end{figure}

\ignore{
\begin{figure}
\centering
\begin{minipage}{.52\textwidth}
  \centering
  \includegraphics[width=\linewidth]{figures/dgan/mnist.png} \\[.05in]
  \includegraphics[width=\linewidth]{figures/dgan/cifar10.png}
\end{minipage}%
\begin{minipage}{0.0\textwidth}
\end{minipage}
\begin{minipage}{.46\textwidth}
  \centering
    \includegraphics[width=\linewidth,,trim= 35 30 40 15,clip=true]{figures/dgan/celeba.png}
\end{minipage}
\caption{Random samples from \dgan\ trained on MNIST, CIFAR10, and CelebA at resolution $128\times128$.
}
\label{fig:dgan_samples}
\end{figure}
}

To enforce the Lipschitz assumption of our Theorems,   
either weight clipping \citep{arjovsky2017wasserstein}, 
\begin{wraptable}{r}{0.48\textwidth}
    \centering
    \caption{
        Inception Score (IS) and Fr\'echet Inception Distance (FID) for
    various datasets.}
    \label{tb:is_fid}
    \begin{tabular}{lccc}
        \toprule
        Dataset     & IS    &FID (train)    &FID (test) \\
        \midrule
        CIFAR10     & 7.02  &26.7           &30.7 \\
        CIFAR100    & 7.31  &28.9           &33.3 \\
        CelebA      & 2.15  &   59.2           & - \\
        \bottomrule
    \end{tabular}
\vskip -.1in
\end{wraptable}
gradient penalization \citep{gulrajani2017improved}, spectral
normalization \citep{miyato2018spectral}, or some combination can be
used.  We found gradient penalization useful for its stabilizing
effect on training, and obtained the best performance with this and
weight clipping.  Table~\ref{tb:is_fid} lists Inception score (IS) and
Fr\'echet Inception distance (FID) on various datasets.  \ignore{ For
  CIFAR10, we obtained an Inception score (IS) of 7.02 and a Fr\'echet
  Inception distance (FID) of 30.7 compared to the test set and 26.7
  compared to the training set. For these sample sizes real data have
  an Inception score of 11.2, and the FID between samples from the
  training and test set is 5.21. For CIFAR100, we report an IS of
  7.31, test FID of 33.3, and train FID of 28.9.  On CelebA, we report
  an IS of 2.15 (compared to 3.61 for real data) and FID of 59.2.}
All results are the best of five trials. While our scores are not
state-of-the-art \citep{highfidelitygan}, they are close to those
achieved by similar unconditional DCGANs
\citep{miyato2018spectral,lucic2018gans}. Figures~\ref{fig:mnist}-\ref{fig:celeba}
show samples from a trained \dgan\ that are not cherry-picked.

\subsection{\edgan}\label{subsec:ensdgan_exp}

\paragraph{Toy example} 

\ignore{\begin{wraptable}{r}{0.48\textwidth}
  \vskip -.15in
  \caption{Likelihood-based metrics of various ensembles of 10 GANs.}
  \centering
  \begin{tabular}{@{\hspace{0cm}}l@{\hspace{.15cm}}c@{\hspace{.15cm}}c@{\hspace{0cm}}}
        \toprule
        & $L(\sdata)$           & $L(\sgan)$ \\ \midrule
        GAN$_1$     &-12.39 ($\pm$ 2.12)    & -796.05 ($\pm$ 12.48) \\
        Ada$_{10}$  &-4.33 ($\pm$ 0.30)     & -266.60 ($\pm$ 24.91) \\
        \edgan$_{10}$   &-3.99 ($\pm$ 0.20)     & -148.97 ($\pm$ 14.13) \\
        \bottomrule
    \end{tabular}
    \label{tb:ens_gmm}
\vskip -.125in
\end{wraptable}
} 
We first considered the toy datasets described in section 4.1 of
AdaGAN \citep{tolstikhin2017adagan}, where we can explicitly compare
various GANs with well-defined, likelihood-based performance metrics.
The true data distribution is a mixture of 9 isotropic Gaussian
components on $\cX = \Rset^2$, with their centers uniformly
distributed on a circle.  We used the AdaGAN algorithm to sequentially
generate 10 GANs, and compared various ensembles of these 10 networks:
GAN$_1$ generated by the baseline GAN algorithm; Ada$_{5}$ and
Ada$_{10}$, generated by AdaGAN with the first 5 or 10 GANs,
respectively; \edgan$_{5}$ and \edgan$_{10}$, the ensembles of the
first 5 or 10 GANs by \edgan, respectively.

\ignore{
\begin{wrapfigure}{r}{.6\textwidth}
\vskip -.2in
\centering
\begin{tabular}{@{\hspace{0cm}}c@{\hspace{.1cm}}c@{\hspace{.1cm}}c@{\hspace{0cm}}}
\includegraphics[width=.33\linewidth,,trim= 35 30 40 15,clip=true]{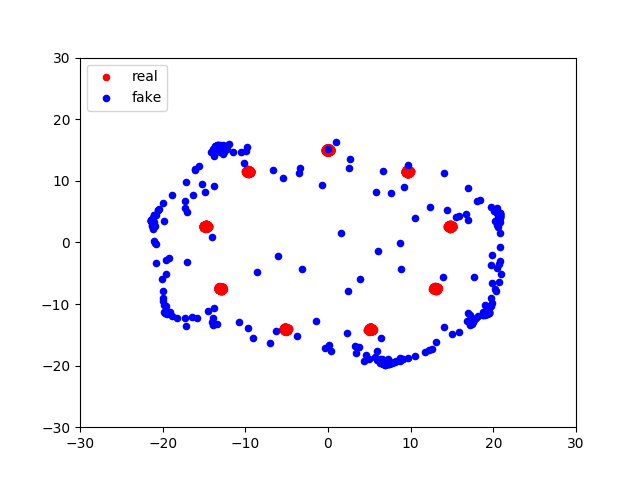} &
\includegraphics[width=.33\linewidth,,trim= 35 30 40 15,clip=true]{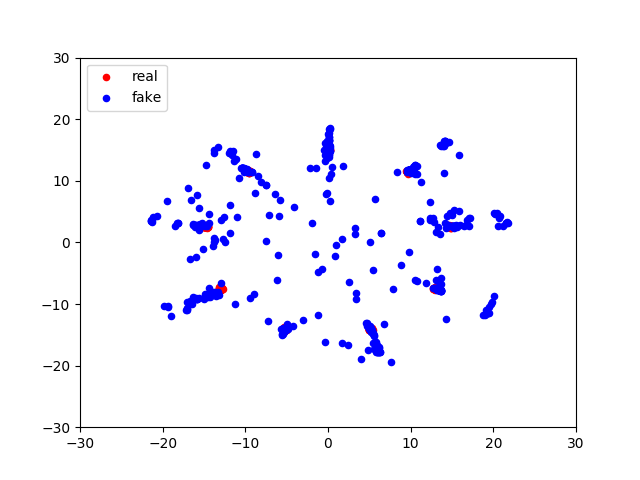} &
\includegraphics[width=.33\linewidth,,trim= 35 30 40 15,clip=true]{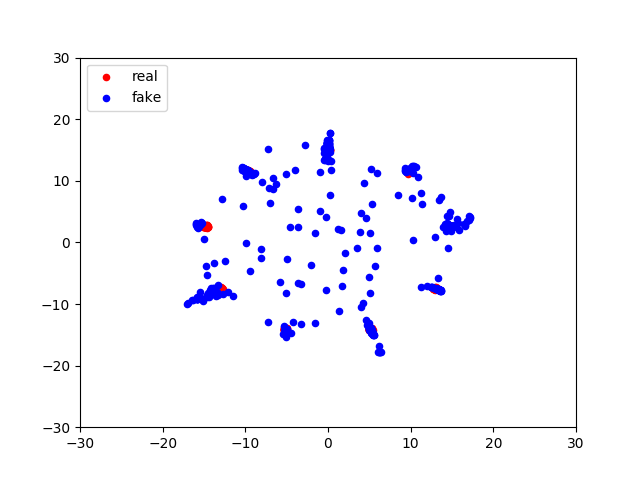}\\
(a) & (b) & (c)\\
\end{tabular}
\caption{The true (red) and the generated (blue)
distributions using various ensembles of 10 GANs. (a) GAN$_1$; (b)
Ada$_{10}$; (c) \edgan$_{10}$.}
\vskip -.15in
\label{fig:ens_gmm}
\end{wrapfigure}
}

The \edgan\ algorithm ran with squared loss and linear mappings.  To
measure the performance, we computed the likelihood of the generated
data under the true distribution $L(\sgan)$, and the likelihood of the
true data under the generated distribution $L(\sdata)$.  We used
kernel density estimation with cross-validated bandwidth to
approximate the density of both $\pgan$ and $\pdata$, as in
\cite{tolstikhin2017adagan}.  We provide part of the ensembles here
and present the full results in Appendix~\ref{app:moreexps}.
Table~\ref{tb:ens_gmm} compares the two likelihood-based metrics
averaged over 10 repetitions, with standard deviation in parentheses.
We can see that for both metrics, ensembles of networks by \edgan\
outperformed AdaGAN using the same number of base networks.
Figure~\ref{fig:ens_gmm} shows the true distribution (in red) and the
generated distribution (in blue).  The single GAN model
(Figure~\ref{fig:ens_gmm}(a)) does not work well.  As AdaGAN gradually
mixes in more networks, the generated distribution is getting closer
to the true distribution (Figure~\ref{fig:ens_gmm}(b)).  By explicitly
learning the mixture weights using discrepancy, \edgan$_{10}$
(Figure~\ref{fig:ens_gmm}(c)) further improves over Ada$_{10}$, such
that the span of the generated distribution is reduced, and the
generated distribution now closely concentrates around the true one.

\begin{figure}
\centering
\begin{minipage}{.4\textwidth}
  \centering
  \captionof{table}{Likelihood-based metrics of various ensembles of 10 GANs.}
  \centering
  \small
  \begin{tabular}{@{\hspace{0cm}}l@{\hspace{.1cm}}c@{\hspace{.1cm}}c@{\hspace{0cm}}}
        \toprule
        & $L(\sdata)$           & $L(\sgan)$ \\ \midrule
        GAN$_1$     &-12.39 ($\pm$ 2.12)    & -796.05 ($\pm$ 12.48) \\
        Ada$_{10}$  &-4.33 ($\pm$ 0.30)     & -266.60 ($\pm$ 24.91) \\
        \edgan$_{10}$   &-3.99 ($\pm$ 0.20)     & -148.97 ($\pm$ 14.13) \\
        \bottomrule
    \end{tabular}
    \label{tb:ens_gmm}
\end{minipage}%
\hfill
\begin{minipage}{.55\textwidth}
\centering
\begin{tabular}{@{\hspace{0cm}}c@{\hspace{.1cm}}c@{\hspace{.1cm}}c@{\hspace{0cm}}}
\includegraphics[width=.33\linewidth,,trim= 100 80 80 100,clip=true]{figures/ens_gmm/gan1.png} &
\includegraphics[width=.33\linewidth,,trim= 100 80 80 100,clip=true]{figures/ens_gmm/ada10.png} &
\includegraphics[width=.33\linewidth,,trim= 100 80 80 100,clip=true]{figures/ens_gmm/ens10.png}\\
(a) & (b) & (c)\\
\end{tabular}
\vskip -.08in
\caption{The true (red) and the generated (blue)
distributions using (a) GAN$_1$; (b)
Ada$_{10}$; (c) \edgan$_{10}$.}
\label{fig:ens_gmm}
\end{minipage}%
\end{figure}

\paragraph{CIFAR10}
We used five pre-trained generators from \cite{lucic2018gans} (all are
publicly available on TF-Hub) as base learners in the ensemble. The
models were trained with different hyperparameters and had different
levels of performance. We then took 50k samples from each generator
and the training split of CIFAR10, and embedded these images using a
pre-trained classifier. We used several embeddings: InceptionV3's
logits layer \citep{szegedy2016rethinking}, InceptionV3's pooling
layer \citep{szegedy2016rethinking}, MobileNet
\citep{sandler2018mobilenetv2}, PNASNet \cite{liu2018progressive},
NASNet \citep{zoph2016neural}, and AmoebaNet
\citep{real2018regularized}. All of these models are also available on
TF-Hub. For each embedding, we trained an ensemble and evaluated its
discrepancy on the test set of CIFAR10 and 10k independent samples
from each generator. We report these results in
Table~\ref{cifar10_ensemble_table}. In all cases \edgan\ performs as
well or better than the best individual generator or a uniform average
of the generators. This also shows that discrepancy generalizes well
from the training to the testing data. Interestingly, depending on
which embedding is used for the ensemble, drastically different
mixture weights are optimal, which demonstrates the importance of the
hypothesis class for discrepancy.  We list the learned ensemble
weights in Table~\ref{cifar10_ensemble_table_weights} in
Appendix~\ref{app:moreexps}.

\begin{table}
    \small
    \caption{Each row uses a different embedding to calculate the discrepancy between the generated images and the CIFAR10 test set.}
    \label{cifar10_ensemble_table}
    \centering
    \vspace{0.1in}
    \begin{tabular}{@{\hspace{0cm}}lllllllll@{\hspace{0cm}}}
        \toprule
        & GAN$_1$ & GAN$_2$ & GAN$_3$ & GAN$_4$ & GAN$_5$ & Best GAN &  Average & \edgan \\\midrule
        InceptionLogits & 285.09      & 259.61      & 259.64      & 271.21      & 272.23      & 259.61         & 259.12          & \textbf{255.3}   \\
        InceptionPool   & 70.52       & 64.37       & 69.48       & 69.69       & 68.7        & 64.37          & 66.08           & \textbf{63.98}   \\
        MobileNet       & 109.09      & 90.47       & 88.01       & 90.9        & 93.08       & 88.01          & 85.71           & \textbf{81.83}   \\
        PNASNet         & 35.18       & 36.42       & 34.94       & 34.38       & 36.52       & 34.38          & 34.66           & \textbf{33.97}   \\
        NASNet          & 54.61       & 52.66       & 59.01       & 61.79       & 64.97       & 52.66          & 55.66           & \textbf{52.46}  \\
        AmoebaNet       & \textbf{97.71}       & 110.83      & 108.61      & 105.31      & 110.5       & 97.71          & 104.91          & \textbf{97.71}   \\\bottomrule
    \end{tabular}
\vskip -.15in
\end{table}

\ignore{
\begin{table}[t]
    \small
    \caption{The mixture weights of each ensemble.}
    \label{cifar10_ensemble_table_weights}
    \centering
    \vspace{0.1in}
    \begin{tabular}{@{\hspace{0cm}}llllll@{\hspace{0cm}}}
        \toprule
        & GAN$_1$ & GAN$_2$ & GAN$_3$ & GAN$_4$ & GAN$_5$ \\\midrule
        InceptionLogits & 0.0007       & 0.4722       & 0.5252       & 0.0009       & 0.0009 \\
        InceptionPool   & 0.0042       & 0.7504       & 0.0139       & 0.0102       & 0.2213 \\
        MobileNet       & 0.0008       & 0.3718       & 0.3654       & 0.2416       & 0.0205 \\
        PNasNet         & 0.3325       & 0.0087       & 0.1400       & 0.5142       & 0.0044 \\
        NasNet          & 0.2527       & 0.7431       & 0.0021       & 0.0012       & 0.0009 \\
        AmoebaNet       & 0.9955       & 0.0005       & 0.0008       & 0.0026       & 0.0006\\\bottomrule
    \end{tabular}
\vskip -.15in
\end{table}
}

\section{Conclusion}

We advocated the use of discrepancy for defining GANs and proved a
series of favorable properties for it, including continuity, under
mild assumptions, the possibility of accurately estimating it from
finite samples, and the generalization guarantees it benefits from.
We also showed empirically that \dgan\ is competitive with other GANs,
and that \edgan, which we showed can be formulated as a convex
optimization problem, outperforms existing GAN ensembles.  For future
work, one can use generative models with discrepancy in adaptation, as
shown in Appendix~\ref{app:gan_adapt}, where the goal is to learn a
feature embedding for the target domain such that its distribution is
close to the distribution of the embedded source domain.  \dgan\ also
has connections with standard Maximum Entropy models (Maxent) as
discussed in Appendix~\ref{app:maxent}.

\subsubsection*{Acknowledgments}

This work was partly supported by NSF CCF-1535987, NSF IIS-1618662,
and a Google Research Award. We thank Judy Hoffman for helpful
pointers to the literature.

\newpage
\bibliography{dgan}
\bibliographystyle{abbrvnat}

\newpage
\appendix
\section{GAN and WGAN}
\label{app:ganintro}

In this appendix section, we briefly introduce and discuss two
instances of the distance metric $d$, which lead to two widely-used
GANs: the original GAN \citep{goodfellow2014generative}, and the WGAN
\citep{arjovsky2017wasserstein}.  Note that in practice, often the
value of $d(\bP,\bQ)$ is not directly computable, and its variational
form is used instead.

\ignore{\paragraph{Notation} Let $\pdata$ and $\pgan$ denote the
  distribution of the true data distribution and the distribution
  generated by GAN, respectively.  GAN is often optimized by
  stochastic gradient descent in batches: at training step $t$, a
  batch of $m$ samples $\sdatat$ is randomly drawn from the true data
  distribution $\pdata$.  Similarly, a batch of $m$ random samples
  $\sgant$ is drawn from the current generator $g_{\theta(t)}$, where
  $\theta(t)$ denotes the parameter of the generator at step $t$.  Let
  $\hdatat$ and $\hgant$ be the empirical distribution of $\sdatat$
  and $\sgant$, respectively.}

\subsection{GAN: Jensen-Shannon divergence}

\cite{goodfellow2014generative} introduced the first GAN framework
using the \emph{Jensen-Shannon} divergence:
\begin{equation*}
d(\pdata,\pgan) \coloneqq \JS (\pdata,\pgan)
    =\big(\,\KL(\pdata\parallel\Pr_{m})+\, \KL(\pgan\parallel\Pr_{m})\,\big)/2,
\end{equation*}
where $\Pr_{m}=(\pdata+\pgan)/2$. The Jensen-Shannon 
divergence admits the following equivalent form:
\begin{equation}\label{eq:JSvar}
\JS(\pdata,\pgan)=\sup_{f\colon\cX\to[0,1]}
\frac{1}{2} \, \Big\{\E_{x\sim \pdata}\big[\log f(x)\big]+
\E_{x\sim \pgan}\big[\log(1-f(x))\big]+\log4\Big\}.
\end{equation}
GANs were originally motivated by expression~\eqref{eq:JSvar}, and its
equivalence to the Jensen-Shannon divergence was shown later on.
Think of $f$ in equation~\eqref{eq:JSvar} as a ``discriminator''
trying to tell apart real data from ``fake'' data generated by
$g_\theta$ as follows: $f$ gives higher scores to samples which it
thinks are real, and gives lower scores otherwise.  Thus the
maximization in~\eqref{eq:JSvar} looks for the best discriminator $f$.
On the other hand, the generator $g_\theta$ tries to fool the
discriminator $f$, such that $f$ cannot tell the difference between
real and fake samples.  Thus minimizing $\JS(\pdata,\pgan)$ over
$\theta$ looks for the best generator $g_\theta$.  Dropping the
constants in~\eqref{eq:JSvar} and parametrizing the discriminator $f$
with a family of neural networks $\{f_{w}\colon\cX\to[0,1], w\in W\}$,
the original algorithm of training GAN \cite{goodfellow2014generative}
considers the following min-max optimization problem:
\begin{equation}\label{eq:obj_GAN}
    \min_{\theta\in\Theta}\max_{w\in W}
    \Big\{\E_{x\sim\pdata}\big[\log f_{w}(x)\big]+
        \E_{x\sim\pgan}\big[\log(1-f_{w}(x))\big]
    \Big\}.
\end{equation}
The generator and the discriminator $(g_\theta, f_w)$ are trained via
stochastic gradient descent/ascent on objective~\eqref{eq:obj_GAN}
with respect to $\theta$ and $w$ iteratively, using the empirical
distributions $\hdata$ and $\hgan$ induced by the batch of samples at
each step.

\begin{remark}
\label{rem:GAN}
When minimizing $\theta$, one can drop the first term in
\eqref{eq:obj_GAN} since it does not depend on $\theta$. Thus, for a
fixed $w$, the minimizing step of $\theta$ is equivalent to
$\min_{\theta}\E_{x\sim\pgan}\big[\log(1-f_{w}(x))\big]$.
\cite{goodfellow2014generative} suggested using maximization instead
of minimization to speed up training for $\theta$. That is, the
training of GAN iterates between
\begin{align*}
    \max_{w\in W} & \,\Big\{\E_{x\sim\pdata}\big[\log f_{w}(x)\big]+
\E_{x\sim\pgan}\big[\log(1-f_{w}(x))\big] \Big\},\\
\max_{\theta\in\Theta} & \, \Big\{\E_{x\sim\pgan}\big[\log f_{w}(x)\big]\Big\}.
\end{align*}
\end{remark}

\subsection{WGAN: Wasserstein Distance}
Instead of minimizing $\JS(\pdata,\pgan)$,
\cite{arjovsky2017wasserstein} proposed to use the \emph{Wasserstein}
distance:
\[
d(\pdata,\pgan)\coloneqq W(\pdata,\pgan)=
\inf_{\gamma\in\Pi(\pdata,\pgan)}\Big\{\E_{(x,y)\sim\gamma}\left\Vert x-y\right\Vert\Big\},
\]
where
$\Pi(\pdata, \pgan)=\{\gamma(x,y)\colon \int_y \gamma(x,y)dy =
\pdata(x), \int_x \gamma(x,y) dx = \pgan(y)\}$ denotes the set of
joint distributions whose marginals are $\pdata$ and $\pgan$,
respectively.  By the Kantorovich-Rubinstein duality
\citep{villani2008optimal}, Wasserstein distance can also be written
as
\begin{equation}\label{eq:wass}
W(\pdata,\pgan)=\sup_{\left\Vert f\right\Vert_{L}\leq 1}
\Big\{\E_{x\sim \pdata}\big[f(x)\big]-\E_{x\sim \pgan}\big[f(x)\big]\Big\},
\end{equation}
where the supremum is taken over all $1$-Lipschitz functions with
respect to the metric $\left\Vert \cdot\right\Vert$ that defines
$W(\pdata,\pgan)$.  Again, we can view $f$ as the discriminator, which
aims to maximizes the difference between its expected values on the
real data and that on the fake data.  In practice,
\cite{arjovsky2017wasserstein} set $f$ to be a neural networks whose
parameters $w$ are limited within a compact set $W$, thus
$\{ f_{w}\colon w\in W\}$ is a set of $K$-Lipschitz functions for some
constant $K$.  Thus, the WGAN (Wasserstein GAN) considered the
following min-max optimization problem:
\begin{equation}\label{eq:obj_WGAN}
\min_{\theta\in\Theta}\max_{w\in W}
\Big\{\E_{x\sim \pdata}\big[f_{w}(x)\big]-
\E_{x\sim \pgan}\big[f_{w}(x)\big]\Big\}.
\end{equation}
The training procedure of WGANs is similar 
to that of GANs \citep{goodfellow2014generative}, where one optimizes 
objective~\eqref{eq:obj_WGAN} using mini-batches with respect to $\theta$ and $w$ iteratively.

\citet{arjovsky2017wasserstein} showed
that the $\JS$ divergence of GAN is potentially not continuous
with respect to generator's parameter $\theta$. On the other hand,
under mild conditions, $W(\pdata,\pgan)$ is continuous
everywhere and differentiable almost everywhere with respect to $\theta$,
making it easier to train WGAN.

When training WGAN with~\eqref{eq:obj_WGAN}, one need to clip the weights $w$ 
to ensure that $w\in W$. 
More recently, \citet{gulrajani2017improved} found that weight clipping
can lead to undesired behavior, such as capacity underuse, and exploding
or vanishing gradients. In view of this, they proposed to add a gradient
penalty to the objective of WGAN, as an alternative to weight clipping:
\[
\min_{\theta\in\Theta}\max_{w}
\Big\{\E_{x\sim \pdata}\big[f_{w}(x)\big]-
\E_{x\sim \pgan}\big[f_{w}(x)\big]+
\lambda\E_{x\sim\Pr_{l}}\big[\left(\left\Vert \nabla_{x}f_{w}\left(x\right)\right\Vert _{2}-1\right)^{2}\big]\Big\},
\]
where $\Pr_{l}$ indicates the uniform distribution along straight
lines between pairs of points in $\sdata$ and $\sgan$. The
construction of $\Pr_{l}$ is motivated by the optimality conditions.

\ignore{For both GAN and WGAN, we can view the supremum over $f$ as learning a
``discriminator'' that tells apart real from generated data, by giving
higher scores to samples which it thinks are real, and giving lower
scores otherwise. This notion of discriminator is a key idea of
GAN. In this paper, we propose to use discrepancy as the distance
metric, which corresponds to using discriminators that depend on
hypothesis set and loss functions used in subsequent supervised
learning tasks. }

WGAN is closely related to \dgan. The definitions of Wasserstein 
distance~\eqref{eq:wass} and discrepancy~\eqref{eq:disc} 
are syntactically the same, except that
the former takes supremum over all 1-Lipschitz functions,
while the latter takes supremum over 
$\ell_{\cH}=\big\{ \ell\big(h(x), h'(x)\big)\colon h,h'\in\cH\big\}$,
a set that depends on the loss and hypothesis set.
Thus, Wasserstein distance can be viewed as discrepancy without 
the hypothesis set and the loss function, which is one reason 
it cannot benefit from theoretical guarantees.

\section{Proofs}
\label{app:proof}
\begin{reptheorem}{th:da}
    Assume the true labeling function $f\colon \cX \to \cY$ is 
    contained in the hypothesis set $\cH$. Then, for any hypothesis $h\in\cH$,
    \[
        \E_{x\sim \pdata} [\ell(h, f)] \leq \E_{x\sim\pgan} [\ell(h, f)] +
        \disc(\pgan, \pdata).
    \]
\end{reptheorem}
\begin{proof}
    By the definition of $\disc$, for any $h\in\cH$,
    \begin{align*}
        \E_{x\sim \pdata} [\ell(h, f)] 
        & \leq \E_{x\sim \pgan} [\ell(h, f)] + 
        \Big|\E_{x\sim \pdata} [\ell(h, f)] -\E_{x\sim \pgan} [\ell(h, f)]  \,\Big| \\
        & \leq \E_{x\sim \pgan} [\ell(h, f)] + 
        \sup_{h,h'\in\cH} \Big|\E_{x\sim \pdata} [\ell(h, h')] -\E_{x\sim \pgan} [\ell(h, h')]  \,\Big| \tag{Since $f\in\cH$}\\
        & = \E_{x\sim \pgan} [\ell(h, f)] + \disc(\pgan, \pdata).
    \end{align*}
\end{proof}

\begin{proposition}\label{prop:lipsz}
    Let $\cH$ be a set of 1-Lipschitz functions. Then, 
$\ell_\cH=\big\{ \big[h(x) - h'(x)\big]^2\colon h,h'\in\cH \big\}$
is a set of Lipschitz functions on $\{x\colon \| x \| \leq 1\}$ 
with Lipschitz constant equals $4$.
\end{proposition}
\begin{proof}
    By definition,
\begin{align*}
    |f(x) - f(x')| &=\big|\big[h(x) - h'(x)\big]^2 - \big[h(x') - h'(x')\big]^2\big| \\
                   & \leq 2\big||h(x) - h'(x)| - |h(x') - h'(x')|\big|  \tag{$\ell_2$ loss is 2-Lipschitz}\\
                   & \leq 2\big|h(x) - h'(x) - h(x') + h'(x')\big| \tag{Triangle inequality}\\
                   & \leq 2\big|h(x) - h(x')| + 2\big| h'(x)  - h'(x')\big| \tag{Triangle inequality}\\
                   &\leq 4\| x-x' \|. \tag{$h$ and $h'$ are 1-Lipschitz}
\end{align*}
\end{proof}

\begin{reptheorem}{th:disc_cont}
Assume that $\cH = \set{h \colon \cX \to \cY}$ is a family of
$\mu$-Lipschitz functions, and the loss function $\ell$ is 
continuous and symmetric in its arguments, and bounded by $M$.
Furthermore, $\ell$ admits the triangle inequality, or
 it can be written as $\ell(y,y') = f(|y-y'|)$ for some Lipschitz
function $f$.  Assume that $g_\theta \colon \cZ \to \cX$ is continuous
in $\theta$.  Then, $\disc(\pdata, \pgan)$ is continuous in
$\theta$.
\end{reptheorem}

\begin{proof}
We first consider the case where $\ell$ admits triangle inequality.
We will show that $\disc(\Pr_{\theta},\Pr_{\theta^{\prime}})\to 0$
as $\theta\to\theta^{\prime}$. By definition of $\Pr_\theta$ and $\disc$,
\begin{align*}
    \disc(\Pr_{\theta},\Pr_{\theta^{\prime}}) 
    & =\sup_{h,h^{\prime}\in\cH} \bigg|\E_{z\sim \pz}\Big[
        \ell\Big(h\big(g_{\theta}(z)\big),
        h^{\prime}\big(g_{\theta}(z)\big)\Big)-
        \ell\Big(h\big(g_{\theta^{\prime}}(z)\big),
        h^{\prime}\big(g_{\theta^{\prime}}(z)\big)\Big)
        \Big]\bigg|\\
    & \leq\sup_{h,h^{\prime}\in\cH} \E_{z\sim\pz}\bigg|
    \ell\Big(h\big(g_{\theta}(z)\big),
    h^{\prime}\big(g_{\theta}(z)\big)\Big)-
    \ell\Big(h\big(g_{\theta^{\prime}}(z)\big),
    h^{\prime}\big(g_{\theta^{\prime}}(z)\big)\Big)\bigg|\\
    & \leq\sup_{h,h^{\prime}\in\cH} \E_{z\sim\pz} \bigg[
        \ell\Big(h\big(g_{\theta}(z)\big),h\big(g_{\theta^{\prime}}(z)\big)\Big)+
        \ell\Big(h^{\prime}\big(g_{\theta}(z)\big),
        h^{\prime}\big(g_{\theta^{\prime}}(z)\big)\Big)\bigg], 
\end{align*}
where we used the triangle inequality and symmetry of $\ell$, 
    such that $\forall a,b,c,d\in\cY$,
\begin{align*}
    |\ell(a,b)-\ell(c,d)| & \leq \ell(a,c)+\ell(b,d).
\end{align*}
Thus,
\begin{align*}
    \disc(\Pr_{\theta},\Pr_{\theta^{\prime}}) 
    \leq \sup_{h\in\cH} 2\E_{z\sim\pz}\bigg[ 
    \ell\Big(h\big(g_{\theta}(z)\big),h\big(g_{\theta^{\prime}}(z)\big)\Big)\bigg] 
    \leq 2\E_{z\sim\pz}\bigg[ 
    \sup_{h\in\cH}\ell\Big(h\big(g_{\theta}(z)\big),
    h\big(g_{\theta^{\prime}}(z)\big)\Big) \bigg].
\end{align*}

Since $\forall h\in\cH$ is $L$-Lipschitz, for any $x_0\in\cX$,
\[
    \lim_{x\to x_{0}}\ell\Big(h(x),h(x_{0})\Big) = 0
\]
converges uniformly over $h\in \cH$.
Furthermore, $g_{\theta}$ is continuous in
$\theta$, it follows that for any fixed $z\in\cZ$,
\[
    \lim_{\theta\to\theta^{\prime}}
    \sup_{h\in\cH}\ell\Big(h\big(g_{\theta}(z)\big),h\big(g_{\theta^{\prime}}(z)\big)\Big)=0,
\]
thus converges point-wise as functions of $z$. 
Since $\ell\leq M$ is bounded, by bounded convergence theorem, we have 
\[
    \lim_{\theta\to\theta^{\prime}} \disc(\Pr_{\theta},\Pr_{\theta^{\prime}})
    \leq 2 \lim_{\theta\to\theta^{\prime}}\E_{z\sim\pz}\bigg[
        \sup_{h\in\cH}\ell\Big(h\big(g_{\theta}(z)\big),
        h\big(g_{\theta^{\prime}}(z)\big)\Big)\bigg] = 0.
\]

Now we consider the case where $ \ell(a, b) = f(|a-b|) $, 
and $f$ is a $q$-Lipschitz function: $ |f(x)  -f(x')| \leq  q |x-x'|$.
By definition,
\begin{align*}
    \disc(\Pr_{\theta},\Pr_{\theta^{\prime}}) 
    & =\sup_{h,h^{\prime}\in\cH} \bigg|\E_{z\sim \pz}\Big[
        \ell\Big(h\big(g_{\theta}(z)\big),
        h^{\prime}\big(g_{\theta}(z)\big)\Big)-
        \ell\Big(h\big(g_{\theta^{\prime}}(z)\big),
        h^{\prime}\big(g_{\theta^{\prime}}(z)\big)\Big)
        \Big]\bigg|\\
    & \leq\sup_{h,h^{\prime}\in\cH} \E_{z\sim\pz}\bigg|
        \ell\Big(h\big(g_{\theta}(z)\big),
        h^{\prime}\big(g_{\theta}(z)\big)\Big)-
        \ell\Big(h\big(g_{\theta^{\prime}}(z)\big),
        h^{\prime}\big(g_{\theta^{\prime}}(z)\big)\Big)\bigg|\\
    & = \sup_{h,h^{\prime}\in\cH} \E_{z\sim\pz} \bigg|
        f\Big(\big|h\big(g_{\theta}(z)\big) - 
        h'\big(g_{\theta}(z)\big)\big|\Big) -
        f\Big(\big|h\big(g_{\theta'}(z)\big) -
        h^{\prime}\big(g_{\theta^{\prime}}(z)\big)\big|\Big)\bigg| \\
    & \leq q\, \sup_{h,h^{\prime}\in\cH} \E_{z\sim\pz} \bigg|
        \Big|h\big(g_{\theta}(z)\big) -
        h'\big(g_{\theta}(z)\big)\Big| -
        \Big|h\big(g_{\theta'}(z)\big) -
        h^{\prime}\big(g_{\theta^{\prime}}(z)\big)\Big|\bigg| \\
    & \leq q\,  \sup_{h,h^{\prime}\in\cH} \E_{z\sim\pz} \bigg|
        h\big(g_{\theta}(z)\big)-h\big(g_{\theta^{\prime}}(z)\big) - 
        h^{\prime}\big(g_{\theta}(z)\big) + 
        h^{\prime}\big(g_{\theta^{\prime}}(z)\big)\bigg| \\
    & \leq q\,  \sup_{h,h^{\prime}\in\cH} \E_{z\sim\pz} \bigg[
        \Big|h\big(g_{\theta}(z)\big) -
        h\big(g_{\theta^{\prime}}(z)\big) \Big| + 
        \Big|h^{\prime}\big(g_{\theta}(z)\big) - 
        h^{\prime}\big(g_{\theta^{\prime}}(z)\big)\Big|\bigg] \\
    & = 2q\, \sup_{h\in\cH} \E_{z\sim\pz} 
        \Big|h\big(g_{\theta}(z)\big)- h\big(g_{\theta'}(z)\big) \Big|  \\
    & \leq 2q\,  \E_{z\sim\pz} \sup_{h\in\cH}
        \Big|h\big(g_{\theta}(z)\big)- h\big(g_{\theta'}(z)\big) \Big|.
\end{align*}

Then, by the same argument above, 
\[
    \lim_{\theta\to\theta^{\prime}} \disc(\Pr_{\theta},\Pr_{\theta^{\prime}})
    \leq 2q \lim_{\theta\to\theta^{\prime}}\E_{z\sim\pz}
    \sup_{h\in\cH}\Big|h\big(g_{\theta}(z)\big),h\big(g_{\theta^{\prime}}(z)\big)\Big| = 0.
\]

Finally, by the triangle inequality of $\disc$,
$\disc(\pdata,\Pr_{\theta})-\disc(\pdata,\Pr_{\theta^{\prime}})
\leq \disc(\pgan,\Pr_{\theta^{\prime}})$,
which completes the proof. 
\end{proof}

\begin{reptheorem}{th:disc_gen}
Assume the loss is bounded, $\ell\leq M$.  
For any $\delta > 0$, with probability at least $1 - \delta$ over the
drawn of $\sdata$ and $\sgan$,
\[ \big|\disc(\pdata,\pgan)-\disc(\hdata,\hgan)\big|
\leq\widehat{\Rad}_{\sdata}(\ell_{\cH})+\widehat{\Rad}_{\sgan}(\ell_{\cH})
+3M\Big(\sqrt{\tfrac{\log(4/\delta)}{2m}}+\sqrt{\tfrac{\log(4/\delta)}{2n}}\,\Big).
\] 
Furthermore, when the loss function $\ell(h,h')$ is a $q$-Lipschitz
function of $h-h'$, we have
\[ 
\big| \disc(\pdata, \pgan) - \disc(\hdata, \hgan)
  \big| \leq 4q \Big( \widehat{\Rad}_{\sdata}(\cH) +
  \widehat{\Rad}_{\sgan}(\cH) \Big) + 
 3M \Big(  \sqrt{\tfrac{\log(4/\delta)}{2m}} + \sqrt{\tfrac{\log(4/\delta)}{2n}} \Big).
\]
\end{reptheorem}
\begin{proof}
By triangle inequality of $\disc(\cdot,\cdot)$,
\[
|\disc(\pdata,\pgan)-\disc(\hdata,\hgan)|\leq\disc(\pdata,\hdata)+\disc(\pgan,\hgan).
\]
We first apply concentration inequality to the scaled loss $\ell_{\cH}/M$:
\begin{align*}
    \E_{\pdata} \ell(h,h')/M  & \leq \E_{\hdata} \ell(h,h')/M + \widehat{\Rad}_{\sdata}(\ell_{\cH}/M) 
    +3\sqrt{\frac{\log(4/\delta)}{2m}},\\
    \E_{\pgan}  \ell(h,h')/M  & \leq \E_{\hgan} \ell(h,h')/M + \widehat{\Rad}_{\sgan}(\ell_{\cH}/M).
\end{align*}
For the empirical Radmacher complexity, we have $\h\Rad_{c\cH} = c\h \Rad_{\cH}$. Thus, we have 
\begin{align*}
    & |\disc(\pdata,\pgan)-\disc(\hdata,\hgan)|\\
    &\leq\disc(\pdata,\hdata)+\disc(\pgan,\hgan)
     \leq\widehat{\Rad}_{\sdata}(\ell_{\cH})+\widehat{\Rad}_{\sgan}(\ell_{\cH})
+3M\bigg(\sqrt{\frac{\log(4/\delta)}{2m}}+\sqrt{\frac{\log(4/\delta)}{2n}}\bigg).
\end{align*}

When the loss function $\ell(h,h')$ is a $q$-Lipschitz function of the difference of its two
arguments, i.e.  $\ell(a, b) = f(a-b)$, and $f(\cdot)$ is a $q$-Lipschitz function,
the mapping of $\cH \ominus \cH \to \ell_{\cH}$ is $q$-Lipschitz,
where $\cH \ominus \cH$ is defined as $\cH \ominus\cH = \{h-h'\colon h,h'\in\cH\}$.
By Talagrand’s contraction lemma, $\h \Rad_{\ell_{\cH}} \leq 2q\Rad_{\cH \ominus \cH}$.
Finally, by definition we have $\h \Rad_{\cH \ominus \cH} \leq 2\h \Rad_{\cH}$.
Putting everything together, when the loss function $\ell(h,h')$ is a $q$-Lipschitz function 
of $h-h'$, 
$$
|\disc(\pdata,\pgan)-\disc(\hdata,\hgan)|\leq4q\Big(\widehat{\Rad}_{\sdata}(\cH)+\widehat{\Rad}_{\sgan}(\cH)\Big)+
3M\bigg(\sqrt{\frac{\log(4/\delta)}{2m}}+\sqrt{\frac{\log(4/\delta)}{2n}}\bigg).
$$

\end{proof}

\begin{reptheorem}{th:ensdgan_gen}
For any $\delta > 0$,
with probability at least
$1 - \delta$ over the draw of samples,
\begin{align*} & |\disc(\Pr_{{\h \balpha}}, \pdata) -
\disc (\Pr_{{\balpha^*}}, \pdata) | \leq
2\Big(\h{\Rad}_{S}(\ell_\cH) +
3M\sqrt{{\log[4(p+1)/\delta]}/{2n}}\,\Big),
\end{align*} 
where
$\h{\Rad}_{S}(\ell_\cH) = \max\big\{ \h
{\Rad}_{S_1}(\ell_{\cH}),\ldots, \h {\Rad}_{S_p}(\ell_{\cH}), \h
{\Rad}_{S_r}(\ell_{\cH}) \big\}$.  Furthermore, when the loss function
$\ell(h,h')$ is a $q$-Lipschitz function of $h - h'$, the following
holds with probability $1 - \delta$:
\begin{align*} & |\disc(\Pr_{{\h \balpha}}, \pdata)-
\disc (\Pr_{{\balpha^*}}, \pdata) | \leq 2\Big( 4q \,
\h{\Rad}_S(\cH) + 3M\sqrt{{\log[4(p+1)/\delta]}/{2n}}\,\Big),
\end{align*} 
where
$\h {\Rad}_S(\cH)= \max\big\{ \h {\Rad}_{S_1}({\cH}),\ldots, \h
{\Rad}_{S_p}({\cH}), \h {\Rad}_{S_r}({\cH}) \big\}$.
\end{reptheorem}
\begin{proof}
    We first extend Theorem~\ref{th:disc_gen} to the case of GAN ensembles:
    \begin{align*}
        |\disc(\pganens, \pdata)-\disc(\hganens, \hdata)|
         \leq \disc(\pganens, \hganens) + \disc(\pdata, \hdata).
    \end{align*}
    For the first term,
    \begin{align*}
        & \disc(\pganens, \hganens) \\
        & \leq \sup_{h,h'\in\cH} 
        \Big|\sum_{k=1}^p \alpha_k \Big\{\E_{x\sim\Pr_k}\big[\ell\big(h(x),h^{\prime}(x)\big)\big]-
        \E_{x\sim\h\Pr_k}\big[\ell\big(h(x),h^{\prime}(x)\big)\big]\Big\}\Big|\\
        & \leq \sup_{h,h'\in\cH} 
        \sum_{k=1}^p \alpha_k \Big|\E_{x\sim\Pr_{k}}\big[\ell\big(h(x),h^{\prime}(x)\big)\big]-
        \E_{x\sim\h\Pr_{k}}\big[\ell\big(h(x),h^{\prime}(x)\big)\big]\Big|\\
        & \leq \sum_{k=1}^p \alpha_k \sup_{h,h'\in\cH} 
        \Big|\E_{x\sim\Pr_{k}}\big[\ell\big(h(x),h^{\prime}(x)\big)\big]-
        \E_{x\sim\h\Pr_{k}}\big[\ell\big(h(x),h^{\prime}(x)\big)\big]\Big|\\
        & = \sum_{k=1}^p \alpha_k \, \disc(\Pr_{k}, \h \Pr_{k}).
    \end{align*}
    By concentration argument, 
    with probability at least $1-\delta$ over the drawn of samples,
    \begin{align*}
        \disc(\pdata,\hdata) & \leq \h{\Rad}_{\sdata}(\ell_{\cH})+
        3M\sqrt{\frac{\log(4(p+1)/\delta)}{2n}}, \\
        \disc(\Pr_{k},\h \Pr_{k}) & \leq \h {\Rad}_{S_k}(\ell_{\cH})+
        3M\sqrt{\frac{\log(4(p+1)/\delta)}{2n}}. 
    \end{align*}

    Putting everything together, with probability at least $1-\delta$, for any $\alpha \in \Delta$,
    \begin{align}\label{eq:ensemblegan_gen}
        & |\disc(\pganens, \pdata)-\disc(\hganens, \hdata)| \nonumber\\
        & \leq \sum_{k=1}^n \alpha_k \, \disc(\Pr_{k}, \h \Pr_{k})
        + \disc(\pdata, \hdata)\nonumber\\
        & \leq \sum_{k=1}^n \alpha_k \, 
        \bigg[\h {\Rad}_{S_k}(\ell_{\cH})+ 3M\sqrt{\frac{\log(4(p+1)/\delta)}{2n}}\bigg]
        + \h{\Rad}_{\sdata}(\ell_{\cH})+ 3M\sqrt{\frac{\log(4(p+1)/\delta)}{2n}} \nonumber\\
        & \leq \h{\Rad}_{S} + 3M\sqrt{\frac{\log(4(p+1)/\delta)}{2n}} .
    \end{align}

    By definition, 
    \begin{align*}
        & \disc (\Pr_{{\balpha^*}}, \pdata) - \disc(\Pr_{{\h \balpha}}, \pdata) \\
        & \leq  \disc (\Pr_{{\balpha^*}}, \pdata) - 
        \disc(\h \Pr_{{\h \balpha}}, \hdata) + 
        |\disc(\h \Pr_{{\h \balpha}}, \hdata)- 
        \disc(\Pr_{{\h \balpha}}, \pdata)|\\
        &\leq  \disc (\Pr_{{\h \balpha}}, \pdata) - 
        \disc(\h \Pr_{{\h \balpha}}, \hdata) +
         |\disc(\h {\Pr}_{{\h \balpha}}, \hdata)- 
        \disc(\Pr_{{\h \balpha}}, \pdata)|\\
        &\leq  2|\disc(\h {\Pr}_{{\h \balpha}}, \hdata)- 
        \disc(\Pr_{{\h \balpha}}, \pdata)|
    \end{align*}
    Similarly,
    \begin{align*}
        & \disc(\Pr_{{\h \balpha}}, \pdata)- \disc (\Pr_{{\balpha^*}}, \pdata)  \\
        \leq & \disc(\Pr_{{\h \balpha}}, \pdata) - \disc (\h\Pr_{{\balpha^*}}, \hdata)   +
        |\disc(\h\Pr_{{\balpha^*}}, \hdata)- \disc(\Pr_{{\balpha^*}}, \pdata)|\\
        \leq &\disc(\Pr_{{\h \balpha}}, \pdata) - \disc (\h\Pr_{{\h \balpha}}, \hdata)   +
        |\disc(\h\Pr_{{\balpha^*}}, \hdata)- \disc(\Pr_{{\balpha^*}}, \pdata)|\\
        \leq & |\disc(\Pr_{{\h \balpha}}, \pdata) - \disc (\h\Pr_{{\h \balpha}}, \hdata)|   +
        |\disc(\h\Pr_{{\balpha^*}}, \hdata)- \disc(\Pr_{{\balpha^*}}, \pdata)|.
    \end{align*}
    Thus, apply inequality~\eqref{eq:ensemblegan_gen} to $\h \balpha$ and $\balpha^*$,
    we have
    \begin{align*}
        & |\disc(\Pr_{{\h \balpha}}, \pdata)- \disc (\Pr_{{\balpha^*}}, \pdata) | 
         \leq 2\Big(\h{\Rad}_{S} + 3M\sqrt{\frac{\log(4(p+1)/\delta)}{2n}}\Big).
    \end{align*}
\end{proof}

\begin{repproposition}{prop:dgan}
When $\ell$ is the squared loss and $\cH$ the family of linear functions with
norm bounded by $1$, 
$\disc (\hdata, \hgan) = 2 \left \| \frac{1}{n} X_\theta^T
  X_\theta - \frac{1}{m} X_r^T X_r \right \|_2$, where $\| \cdot \|_2$
denotes the spectral norm.
\end{repproposition}
\begin{proof}
\begin{align*}
    \disc (\hdata,\hgan) 
   =  & \sup_{\substack{\parallel w \parallel_2 \leq 1\\\parallel w' \parallel_2 \leq 1}}
    \Big|\E_{x\sim\hdata}\big[\ell_2\big(w^T x, w'^{T}x\big)-
    \E_{x\sim\hgan}\big[\ell_2\big(w^T x, w'^{T}x\big)\Big| \\
     = & \sup_{\substack{\parallel w \parallel_2 \leq 1\\\parallel w' \parallel_2 \leq 1}}
    \Big| \frac{1}{n}(X_\theta w-X_\theta w')^T (X_\theta w-X_\theta w')
    - \frac{1}{m}(X_r w-X_r w')^T (X_r w-X_r w') \Big| \\
     = &  \sup_{\parallel u \parallel_2 \leq 2} \Big|
      \frac{1}{n} u^T X_\theta^T X_\theta u  - \frac{1}{m} u^T X_r ^T X_r u   \Big| \tag{Let $u=w-w'$}\\
     = &  2\sup_{\parallel u \parallel_2 \leq 1} \Big| u^T \Big(\frac{1}{n}X_\theta^T X_\theta 
      - \frac{1}{m} X_r^T X_r\Big) u  \Big|\\
      = & 2 \left\|\frac{1}{n}X_\theta^T X_\theta 
      - \frac{1}{m} X_r^T X_r  \right\|_2.
  \end{align*}
\end{proof}

\section{More Experiments}\label{app:moreexps}

\subsection{\edgan: Toy datasets}
In this section, we provide more results on mixing the 10 GANs generated by AdaGAN.
Recall that we are comparing the following methods:
\begin{itemize}
\item The baseline GAN algorithm, namely GAN$_1$.
\item The AdaGAN algorithm, ensembles of the first 5 GANs, namely Ada$_{5}$.
\item The AdaGAN algorithm, ensembles of the first 10 GANs, namely Ada$_{10}$.
\item The \edgan\ algorithm, ensembles of the first 5 GANs, namely \edgan$_{5}$.
\item The \edgan\ algorithm, ensembles of the first 10 GANs, namely \edgan$_{10}$.
\end{itemize}

We considered two ways of computing average sample log-likelihood
and used them as performance metrics:
the likelihood of the generated data under the true distribution $L(\sgan)$,
and the likelihood of the true data under the generated distribution $L(\sdata)$.
To be more concrete, 
$$L(\sgan) = L_{\pdata}(\sgan) = \frac{1}{N} \sum_{x_i\in \sgan} \log \big(\pdata(x_i)\big), \quad
L(\sdata) = L_{\pgan}(\sdata) = \frac{1}{N} \sum_{x_i\in \sdata} \log \big(\pgan(x_i)\big).$$
We used kernel density estimation with cross-validated
bandwidth to approximate the density of both $\pgan$ and $\pdata$.

Figure~\ref{fig:app_ens_gmm}
displayed the true distribution (in red) and the generated distribution 
under various ensembles of GANs. 
\edgan$_5$ and \edgan$_{10}$ improve the generated distribution
over Ada$_5$ and Ada$_{10}$, respectively.

\begin{figure}[t]
    \centering
    \begin{subfigure}{.32\textwidth}
        \centering
        \includegraphics[width=1\linewidth]{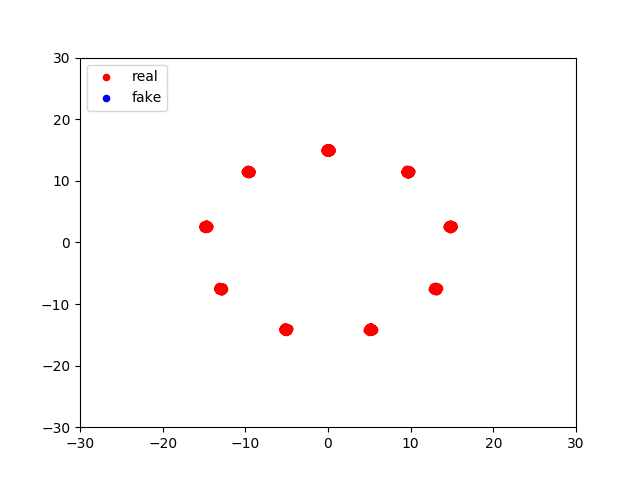}
        \caption{$\pdata$}
        \label{subfig:app_ens_gmm_real}
    \end{subfigure}
    \begin{subfigure}{.32\textwidth}
        \centering
        \includegraphics[width=1\linewidth]{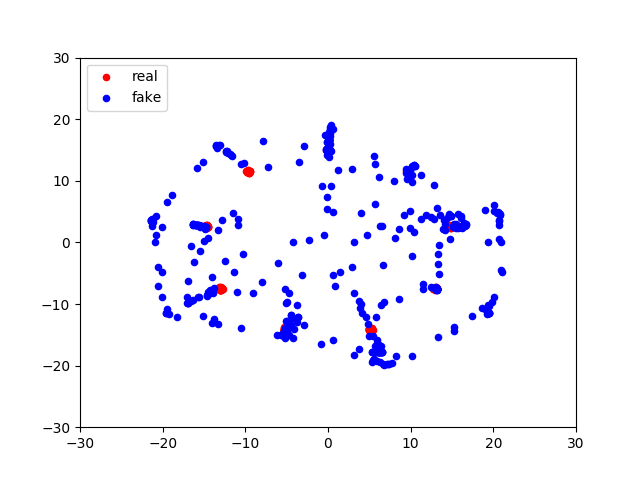}
        \caption{Ada$_5$}
        \label{subfig:app_ens_gmm_ada5}
    \end{subfigure}
    \begin{subfigure}{.32\textwidth}
        \centering
        \includegraphics[width=1\linewidth]{figures/ens_gmm/ada10.png}
        \caption{Ada$_{10}$}
        \label{subfig:app_ens_gmm_ada10}
    \end{subfigure}\\
    \begin{subfigure}{.32\textwidth}
        \centering
        \includegraphics[width=1\linewidth]{figures/ens_gmm/gan1.png}
        \caption{GAN$_1$}
        \label{subfig:app_ens_gmm_gan1}
    \end{subfigure}
    \begin{subfigure}{.32\textwidth}
        \centering
        \includegraphics[width=1\linewidth]{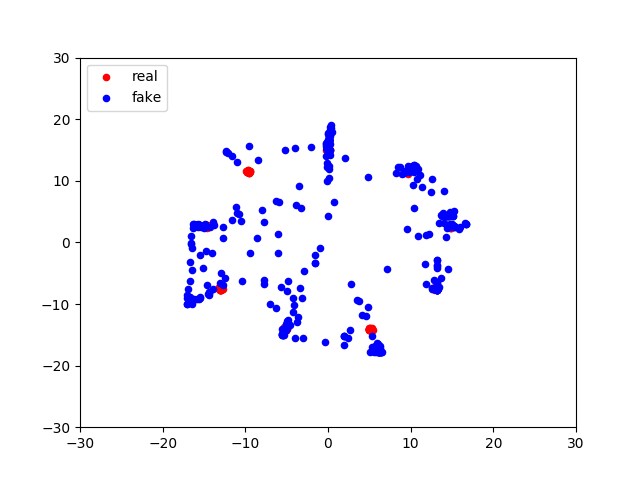}
        \caption{\edgan$_5$ }
        \label{subfig:app_ens_gmm_ens5}
    \end{subfigure}
    \begin{subfigure}{.32\textwidth}
        \centering
        \includegraphics[width=1\linewidth]{figures/ens_gmm/ens10.png}
        \caption{\edgan$_{10}$}
        \label{subfig:app_ens_gmm_ens10}
    \end{subfigure}
    \caption{The true (red) and the generated (blue) distributions, using various ensembles 
    of 10 GANs.}
    \label{fig:app_ens_gmm}
\end{figure}

Table~\ref{tb:app_ens_gmm} showed the average log-likelihoods over 10 repetitions, 
with standard deviations in parentheses, where a higher log-likelihood indicates
better performance.  
We can see that for both metrics, networks ensembles \edgan$_5$ and \edgan$_{10}$ by \edgan\ 
outperformed AdaGAN with the same number of base networks.

\begin{table}[h]
\caption{Likelihood-based metrics of various ensembles of 10 GANs.}
\centering
    \vspace{0.1in}
\begin{tabular}{ccc}
    \toprule
            & $L(\sdata)$           & $L(\sgan)$ \\ \midrule
GAN$_1$     &-12.39 ($\pm$ 2.12)    & -796.05 ($\pm$ 12.48) \\
Ada$_{5}$   &-5.02 ($\pm$ 0.11)     & -296.45 ($\pm$ 15.24) \\
Ada$_{10}$  &-4.33 ($\pm$ 0.30)     & -266.60 ($\pm$ 24.91) \\
\edgan$_{5}$    &-4.85 ($\pm$ 0.16)     & -172.52 ($\pm$ 17.56) \\
\edgan$_{10}$   &-3.99 ($\pm$ 0.20)     & -148.97 ($\pm$ 14.13) \\
    \bottomrule
\end{tabular}
\label{tb:app_ens_gmm}
\end{table}

\subsection{\edgan: CIFAR10}

In this section, we provide the mixture weights of each ensemble when
learning \edgan\ on CIFAR10 generators, as described in
Section~\ref{subsec:ensdgan_exp}. The data is provided in
Table~\ref{cifar10_ensemble_table_weights}. Only in one instance a
significant amount of the weight is allocated to one model.
\begin{table}[t!]
    \caption{The mixture weights of each ensemble.}
    \label{cifar10_ensemble_table_weights}
    \centering
    \vspace{0.1in}
    \begin{tabular}{@{\hspace{0cm}}llllll@{\hspace{0cm}}}
        \toprule
        & GAN$_1$ & GAN$_2$ & GAN$_3$ & GAN$_4$ & GAN$_5$ \\\midrule
        InceptionLogits & 0.0007       & 0.4722       & 0.5252       & 0.0009       & 0.0009 \\
        InceptionPool   & 0.0042       & 0.7504       & 0.0139       & 0.0102       & 0.2213 \\
        MobileNet       & 0.0008       & 0.3718       & 0.3654       & 0.2416       & 0.0205 \\
        PNasNet         & 0.3325       & 0.0087       & 0.1400       & 0.5142       & 0.0044 \\
        NasNet          & 0.2527       & 0.7431       & 0.0021       & 0.0012       & 0.0009 \\
        AmoebaNet       & 0.9955       & 0.0005       & 0.0008       & 0.0026       & 0.0006\\\bottomrule
    \end{tabular}
\vskip -.15in
\end{table}

\subsection{Addition \dgan\ experimental details}
All experiments for \dgan\ used Adam. On MNIST, we trained for 200 epochs at batch size 32 with learning rates of $3 \times 10^{-5}$ for the generator and $1 \times 10^{-5}$ for discriminator. For CIFAR10 and CIFAR 100, we trained for 256 epochs at batch size 32 with learning rates of $3 \times 10^{-5}$ for the generator and $1 \times 10^{-5}$ for discriminator.  For CelebA, larger batch sizes and learning rates were necessary: batch size 256 and learning rates of $2 \times 10^{-4}$ for the generator and $5 \times 10^{-4}$ for discriminator.

\begin{table}[t!]
    \caption{DGAN architectures based on \cite{miyato2018spectral}. Let $b$ be the batch size, $(h, w, c)$ be the shape of an input image, $f$ be the width of the discriminator's output embedding, and $g=4$ for CIFAR and $g=16$ for CelebA. The italicized layers in the discriminators are skipped for CIFAR (resulting in a shallower model), but are included for CelebA.}
    \label{dgan_architectures}
    \vspace{0.1in}
    \begin{tabular}{@{\hspace{-0.2cm}} l P{5.1cm} l @{\hspace{-0.2cm}}}
        \toprule
        &\textbf{Discriminator}&\\\toprule
        &Images $x\in \mathbb{R}^{b\times h\times w\times c}$&\\\midrule
        &$3\times 3$, stride 1, conv. 64, BN, ReLU&\\
        &$4\times 4$, stride 2, conv. 64, BN, ReLU&\\\midrule
        &$3\times 3$, stride 1, conv. 128, BN, ReLU&\\
        &$4\times 4$, stride 2, conv. 128, BN, ReLU&\\\midrule
        &\emph{3$\times $3, stride 1, conv. 256, BN, ReLU}&\\
        &\emph{4$\times $4, stride 2, conv. 256, BN, ReLU}&\\\midrule
        &$3\times 3$, stride 1, conv. 512, BN, ReLU&\\\midrule
        &dense$\to\mathbb{R}^{b\times f}$&\\\bottomrule
    \end{tabular}
    \qquad
    \begin{tabular}{@{\hspace{-0.2cm}} l P{5.5cm} l @{\hspace{-0.2cm}}}
        \toprule
        &\textbf{Generator}&\\\toprule
        &noise $z\in \mathbb{R}^{b\times 128}$&\\\midrule
        &dense$\to g^2\times 512$&\\\midrule
        &$4\times 4$, stride 2, deconv. 256, BN, ReLU&\\\midrule
        &$4\times 4$, stride 2, deconv. 128, BN, ReLU&\\\midrule
        &$4\times 4$, stride 2, deconv. 64, BN, ReLU&\\\midrule
        &$3\times 3$, stride 1, conv. 3, Tanh&\\\bottomrule
    \end{tabular}
\vskip -.15in
\end{table}

\section{Domain Adaptation}\label{app:gan_adapt}
We first introduce additional notation for the domain adaptation task. 
Let $\sD_s$ and $\sD_t$ denote the source and target distribution over $\cX \times \cY$, 
and let $\h \sD_s$ and $\h \sD_t$ denote the empirical distribution induced by samples 
drawn according to $\sD_s$ and $\sD_t$, respectively.  
For any distribution $\sD$, denote by $\sD^\cX$ its marginal distribution on the input space $\cX$.
Finally, for any marginal distribution $\sD^\cX$ and a feature mapping $M\colon \cX \to \cZ$
that maps input space $\cX$ to some feature space $\cZ$, 
we denote by $M(\sD^\cX)$ the distribution of $M(x)$ where $x\sim \sD^\cX$.

\subsection{Adversarial Discriminative Domain Adaptation (ADDA)}
\cite{tzeng2017adversarial} considered a domain adaptation framework, 
Adversarial Discriminative Domain Adaptation (ADDA),
which has a very similar motivation to GANs. Given a pre-trained source domain feature mapping $M_s$, 
ADDA simultaneously optimizes a target domain feature mapping $M_t$ and an adversarial discriminator,
such that the best discriminator cannot tell apart the mapped features from source and target domain.
At test time, ADDA applies the classifier trained on source feature mapping and labels 
to the learned target feature mapping, to predict target label. 
Take the multi-class classification task for example, the ADDA consists of three stages:
\begin{enumerate}
    \item Pre-training. Given labeled samples $(X_s,Y_s)$ from source domain, 
        learn a source feature mapping $M_s$ and a classifier $C$ under cross-entropy loss:
        \[
            \min_{M_s, C} \Big\{-\E_{(x_s,y_s) \sim (X_s,Y_s)} 
            \sum_{k=1}^K 1_{k=y_s} \log C(M_s(x_s))\Big\},
        \]
        where $K$ is the number of label classes.
\item Adversarial adaptation. Given pre-trained source feature mapping $M_s$ and
    unlabeled samples $X_t$ from target domain, 
    jointly learn a discriminator $D$ and a target feature mapping $M_t$:
    \begin{align*}
        \min_{D} &  \Big\{-\E_{x_s \sim X_s} \big[\log D(M_s(x_s))\big] - 
    \E_{x_t \sim X_t} \big[\log (1-D(M_t(x_t)))\big]\Big\}, \tag{Learn $D$} \\
    \min_{M_t} & \Big\{-\E_{x_t \sim X_t} \big[\log D(M_t(x_t))\big]\Big\}. \tag{Learn $M_t$}
    \end{align*}
\item Testing. Predict label for target data based on $C(M_t(x_t))$.
\end{enumerate}
Note that the second stage (adversarial adaptation) is very similar to the GAN framework, 
where the discriminator has the same functionality, and the generator is now mapping 
from the target data, instead of from a random latent variable, to a desired feature space.

The key idea of ADDA is very similar to GAN:
the adversarial training step is in fact minimizing the Jensen-Shannon divergence
between the mapped source distribution and the mapped target distribution:
\begin{align*}
    \min_{M_t} \quad \JS\Big(M_s(\sD_s^\cX), M_t(\sD_t^\cX)\Big).
\end{align*}
Since discrepancy is originally designed for domain adaptation, 
it is natural to use discrepancy as the distance metric, instead of the Jensen-Shannon divergence, 
in this ADDA framework. We give more details below.

\subsection{\dgan\ for Domain Adaptation}
The procedure of ADDA with discrepancy is very similar to the original ADDA,
which is described below.
\begin{enumerate}
    \item Pre-training. Given labeled samples from source domain, or equivalently 
        an empirical distribution $\h \sD_s$, 
        learn a source feature mapping $M_s$ and a classifier 
        $\h h_s\in \cH$:
        \[
            M_s, \h h_s = \argmin_{M, h} 
            \Big\{\E_{(x,y) \sim \h \sD_s} 
           \Big[ \ell\Big(h(M(x)),y\Big)\Big] \Big\}
        \]
\item Adversarial adaptation. Given pre-trained source feature mapping $M_s$ and
    unlabeled samples from target domain (or equivalently, $\h \sD_t^\cX$), 
    learn a target feature mapping $M_t$, such that the distribution of 
        $M_s(\h \sD_s^\cX)$ and $M_t(\h \sD_t^\cX)$ are small under discrepancy:
        \begin{align}\label{eq:disc_gan_adapt}
            M_t & =  \argmin_{M}  \Big\{ \disc
            \Big(M_s(\h \sD_s^\cX), M(\h \sD_t^\cX)\Big) \Big\} \nonumber \\
            & = \argmin_{M}  \Big\{
            \sup_{h,h^{\prime}\in\cH} 
            \Big| \E_{z \sim M_s(\h \sD_s^\cX)}
            \Big[\ell\Big( h(z), h^{\prime}(z)\Big)\Big] - 
             \E_{z \sim M(\h \sD_t^\cX)}
            \Big[\ell\Big( h(z), h^{\prime}(z)\Big)\Big]\Big|
        \Big\}.  
    \end{align}

\item Testing. Predict label for target data using $\h h_s (M_t(\cdot))$.
\end{enumerate}

Suppose the target mapping $M_t$ is parameterized by and continuous in
$\theta$.  Then, under the same assumptions of
Theorem~\ref{th:disc_cont}, the objective function
in~\eqref{eq:disc_gan_adapt} is continuous in $\theta$.

To analyze the adaptation performance, for the fixed mapping $M_s$ and
$M_t$, we define the risk minimizers $h_s^*$ and $h_t^*$:
\begin{align*}
    h_s^*  &= \argmin_{h\in \cH}  \E_{(x,y) \sim \sD_s}  
    \Big[\ell\Big(h(M_s(x)),y\Big)\Big],\quad
h_t^*  = \argmin_{h\in \cH}  \E_{(x,y) \sim \sD_t}  
    \Big[\ell\Big(h(M_t(x)),y\Big)\Big].
\end{align*}

We have the following learning guarantees for ADDA with discrepancy.
\begin{theorem}\label{thm:dgan-da}
Assume the loss function $\ell(\cdot,\cdot)$ is symmetric and obeys
triangle inequality. Then,
\begin{align*}
    & \E_{(x,y)\sim\sD_t} \Big[\ell\Big(\h h_s(M_t(x)),y\Big)\Big]\\
    & \leq \E_{x\sim\sD_s^\cX} \Big[\ell\Big(\h h_s(M_s(x)), h_s^*(M_s(x))\Big)\Big]
    + \disc\Big(M_t(\sD_t^\cX), M_s(\sD_s^\cX)\Big) \\
    & + \E_{x\sim\sD_t^\cX} \Big[\ell\Big(h_s^*(M_t(x)), h_t^*(M_t(x))\Big)\Big]
    +\E_{(x,y)\sim\sD_t} \Big[\ell\Big(h_t^*(M_t(x)),y\Big) \Big].
\end{align*}
\end{theorem}
\begin{proof}
    By triangle inequality of $\ell$, we have
\begin{align*}
    & \E_{(x,y)\sim\sD_t} \Big[\ell\Big(\h h_s(M_t(x)),y\Big)\Big]\\
    & \leq \E_{(x,y)\sim\sD_t} \Big[\ell\Big(\h h_s(M_t(x)), h_s^*(M_t(x))\Big)\Big]\\
    & + \E_{(x,y)\sim\sD_t} \Big[\ell\Big(h_s^*(M_t(x)), h_t^*(M_t(x))\Big)\Big]
     +\E_{(x,y)\sim\sD_t} \Big[\ell\Big(h_t^*(M_t(x)),y\Big)\Big] \\
    & \leq \E_{x\sim\sD_s^\cX} \Big[\ell\Big(\h h_s(M_s(x)), h_s^*(M_s(x))\Big)\Big]
    + \disc\Big(M_t(\sD_t^\cX), M_s(\sD_s^\cX)\Big) \\
    & + \E_{x\sim\sD_t^\cX} \Big[\ell\Big(h_s^*(M_t(x)), h_t^*(M_t(x))\Big)\Big]
    +\E_{(x,y)\sim\sD_t} \Big[\ell\Big(h_t^*(M_t(x)),y\Big) \Big]
\end{align*}
\end{proof}

Let us examine each item in Theorem~\ref{thm:dgan-da}:
\begin{itemize}
\item The first term is the estimation error of $\h h_s$, which should
  be small when a large set of source data is available.
\item The second term is the true discrepancy between $M_t(\sD_t^\cX)$
  and $M_s(\sD_s^\cX)$.  According to Theorem~\ref{th:disc_gen}, it
  can be accuracy estimated by its empirical counterparts, which is
  minimized during the training step
  (Equation~\eqref{eq:disc_gan_adapt}).

\item The third term depends on how different $h_s^*$ and $h_t^*$ are,
  and it is essentially determined by how difficult the adaption
  problem is.

\item The last term is the minimal error achievable by $\cH$ with
  feature mapping $M_t$ on the target domain.  When $\cH$ is a complex
  family of hypothesis, such as neural networks, this term can be
  viewed as a lower bound of the adaptation performance, and it is
  determined by how difficult the learning problem on the target
  domain is.
\end{itemize}
Therefore, the only term we have control over is the discrepancy term,
and thus by minimizing the discrepancy during training, we are
reducing the upper bound on the adaptation performance.  This
validates the use of discrepancy in ADDA.

\section{Connection Between \dgan\ and Maxent}
\label{app:maxent}

Both \dgan\ and maximum entropy (Maxent) are methods for density
estimation. In this section we show that Maxent is a regularized
version of \dgan.

Let $\Delta$ denote the simplex of all probability distributions over
$\cX$, and let ${\bPhi}:\cX\to \mathbb{R}^d$ be the feature mapping.
The maximum entropy (Maxent) model for density estimation solves the
following optimization problem:
\begin{align*}
    \max_{\mathbb{P} \in \Delta}  \, \mathtt{H}(\mathbb{P}), \quad
    \text{s.t.}  \, \Big\| \E_{x\sim \mathbb{P}} [{\bPhi}(x)]-
    \E_{x\sim \hdata} [{\bPhi}(x)] \Big\|_{\infty} \leq \lambda,
\end{align*}
where $\bPhi = (f_1, f_2, \ldots , f_n)$, and $\cF=\{f_i, i\in[n]\}$
is the set of feature functions, $f_i \colon \cX \to \mathbb{R}$.
Note that $\max_{\mathbb{P} \in \Delta} \, \mathtt{H}(\mathbb{P})$ is
equivalent to
$\max_{\mathbb{P} \in \Delta} \, \KL(\mathbb{P}\parallel \hdata)$.

To see the connection between Maxent and \dgan, 
we can set $\cF=\{\ell_{h,h'} \colon \ell_{h,h'}(x) = \ell(h(x),h'(x)), h,h'\in \cH\}$,
where $\cH$ is the hypothesis set that defines  the discrepancy $\disc$. Then 
the Maxent optimization problem becomes 
\begin{align}\label{opt:maxent}
    \max_{\mathbb{P} \in \Delta}  \, \KL(\mathbb{P}\parallel \hdata), \quad
    \text{s.t.}  \, \max_{h,h' \in \cH} \,  
    \Big|\E_{x\sim \mathbb{P}} \big[\ell\big(h(x),h'(x)\big)\big]-
    \E_{x\sim \hdata} \big[\ell\big(h(x),h'(x)\big)\big] \Big|  \leq \lambda.
\end{align}
In fact, we can write the dual problem of~\eqref{opt:maxent} as
\begin{align}\label{opt:maxent_dual}
  \min_{\mathbb{P} \in \Delta}  \, - \KL(\mathbb{P}\parallel \hdata)  
  + \alpha \bigg\{
  \max_{h,h' \in \cH} \,  \Big|\E_{x\sim \mathbb{P}} \big[\ell\big(h(x),h'(x)\big)\big]- \E_{x\sim \hdata} \big[\ell\big(h(x),h'(x)\big)\big] \Big|  \bigg\},
\end{align}
where $\alpha\ge0$ is the Lagrange multiplier.

Recall that \dgan\ solves the following optimization problem:
\ignore{\begin{align}\label{opt:dgan}
    \min_{\mathbb{P} \in \{\pgan: \theta\in\Theta\}} \, \max_{f \in \cF} \, \Big\{\E_{x\sim \mathbb{P}} [f(x)]-
    \E_{x\sim \hdata} [f(x)]  \Big\}, 
\end{align}}
\begin{equation}\label{opt:dgan}
    \min_{\mathbb{P} \in \{\pgan: \theta\in\Theta\}} \max_{h,h'\in\cH}
    \Big|\E_{x\sim\mathbb{P}}\big[\ell\big(h(x),h'(x)]\big)-
        \E_{x\sim\hdata}\big[\ell\big(h(x),h'(x)\big) \Big|,
\end{equation}
where $\{\pgan: \theta\in\Theta\}$ is a parametric family of
distribution, $\{\pgan: \theta\in\Theta\} \subseteq \Delta$.  Thus,
the dual problem of Maxent~\eqref{opt:maxent_dual} can be viewed as
\dgan~\eqref{opt:dgan}, plus a regularization term in the form of KL
divergence $\KL(\mathbb{P}\parallel \hdata)$.

However, to use~\eqref{opt:maxent_dual} under the \dgan\ framework,
$\mathbb{P}$ is optimized over the special parametric family of
distributions $\{\pgan: \theta\in\Theta\}$.  The probability density
of $\pgan(x)$ at any $x\in\cX$ is unavailable, and thus we cannot
directly compute the KL divergence $\KL(\pgan \parallel \hdata)$.  One
option is to use the Donsker-Varadhan \citep{donsker1975asymptotic}
representation:
\begin{align*}
    \KL(\pgan \parallel \hdata) = \sup_{f\colon \cX \to \mathbb{R}}  \, 
    \E_{x \sim \pgan}[f(x)] - \log\Big(\E_{x\sim \hdata} [e^{f(x)}]\Big).
\end{align*}

Putting everything together, we get the regularized \dgan\ formulation that is 
motivated by the Maxent model: for some $\alpha>0$,
\begin{align}\label{eq:obj_discGAN_maxent}
    \min_{\theta}  \, & \bigg\{ \inf_{f\colon \cX \to \mathbb{R}}  \, 
     \log\Big(\E_{x\sim \hdata} [e^{f(x)}]- \E_{x \sim \pgan}[f(x)]  \Big)\bigg\} \nonumber \\
    & + \alpha \bigg\{
    \max_{h,h' \in \cH} \,  \Big|\E_{x\sim \mathbb{P}} \big[\ell\big(h(x),h'(x)\big)\big]- \E_{x\sim \hdata} \big[\ell\big(h(x),h'(x)\big)\big] \Big|  \bigg\}.
\end{align}

\end{document}